\pdfoutput=1
\documentclass{article}

    \PassOptionsToPackage{numbers, compress}{natbib}

 \usepackage[main, final]{neurips_2026}

\usepackage[utf8]{inputenc} %
\usepackage[T1]{fontenc}    %
\usepackage{hyperref}       %
\usepackage{url}            %
\usepackage{booktabs}       %
\usepackage{amsfonts}       %
\usepackage{nicefrac}       %
\usepackage{microtype}      %
\usepackage{xcolor}         %
\usepackage{acro}           %
\usepackage[numbers]{natbib}
\usepackage{amsmath}
\usepackage{caption}
\usepackage{graphicx}
\usepackage{wrapfig}
\usepackage[table]{xcolor}
\usepackage{changepage}
\usepackage{algorithm}
\usepackage{algpseudocode}
\usepackage{geometry}
\usepackage{amsthm}
\usepackage{mathtools}
\usepackage{adjustbox}
\usepackage{makecell}
\usepackage{multirow}
\definecolor{csBlue}{RGB}{51, 187, 238}
\definecolor{csGreen}{RGB}{0, 153, 136}
\definecolor{csOrange}{RGB}{238, 119, 51}
\definecolor{csRed}{RGB}{204, 51, 17}
\definecolor{csGrey}{RGB}{187, 187, 187}
\newcommand{\sameboth}{\textcolor{csGrey}{\scriptsize[same as FQL/floq]}}
\newcommand{\samefql}{\textcolor{csBlue}{\scriptsize[same as FQL]}}
\newcommand{\ourspec}{\textcolor{csOrange}{\scriptsize[FlowIQN-specific]}}
\newcommand{\shortcut}{\textcolor{csRed}{\scriptsize[shortcut specific]}}
\newcommand{\samefloq}{\textcolor{csGreen}{\scriptsize[same as floq]}}

\makeatletter
\@ifundefined{c@savedNested}{\newcounter{savedNested}}{}
\newcommand{\saveALGstate}{%
  \setcounter{savedNested}{\value{ALG@nested}}%
  \expandafter\global\expandafter\let\expandafter\savedALGblock\csname ALG@currentblock@\arabic{ALG@nested}\endcsname
  \expandafter\global\expandafter\let\expandafter\savedALGlife\csname ALG@currentlifetime@\arabic{ALG@nested}\endcsname
}
\newcommand{\restoreALGstate}{%
  \expandafter\global\expandafter\let\csname ALG@currentblock@\arabic{savedNested}\endcsname\savedALGblock
  \expandafter\global\expandafter\let\csname ALG@currentlifetime@\arabic{savedNested}\endcsname\savedALGlife
  \setcounter{ALG@nested}{\value{savedNested}}%
}
\makeatother

\definecolor{cbBlue}{HTML}{56B4E9}

\usepackage{tikz}

\DeclareAcronym{RL}{short=RL,long=Reinforcement Learning}
\DeclareAcronym{DRL}{short=DRL,long=Distributional Reinforcement Learning}
\DeclareAcronym{CFM}{short=CFM,long=Conditional Flow Matching}
\DeclareAcronym{ODE}{short=ODE,long=Ordinary Differential Equation}

\title{Quantile-Coupled Flow Matching for Distributional Reinforcement Learning}

\author{%
Michael Groom$^{1,*}$, 
Victor-Alexandru Darvariu$^{2,\dagger}$,
Lars Kunze$^{1,4}$, \\
\textbf{James Wilson$^{3}$,
Nick Hawes$^{1}$}  \\
$^{1}$Oxford Robotics Institute, University of Oxford, Oxford, UK \\
$^{2}$School of Electrical \& Electronic Engineering, University of Sheffield, Sheffield, UK \\
$^{3}$Dyson Institute of Engineering \& Technology, Malmesbury, UK \\
$^{4}$Bristol Robotics Laboratory, University of the West of England, Bristol, UK \\
$^{*}$Corresponding author: \texttt{michaelgroom@robots.ox.ac.uk}. \\
$^{\dagger}$Work performed while the author was affiliated with the Oxford Robotics Institute.
}

\newcommand{\beforesec}{\vspace{-6pt}}
\newcommand{\aftersec}{\vspace{-4pt}}
\newcommand{\beforesubsec}{\vspace{-4pt}}
\newcommand{\aftersubsec}{\vspace{-4pt}}

\newtheorem{proposition}{Proposition}
\newtheorem{corollary}{Corollary}
\newtheorem{definition}{Definition}
\begin{document}

\maketitle

\begin{abstract}

  Distributional Reinforcement Learning (DRL) models the full return distribution, making it well-suited to uncertainty-aware and risk-sensitive decision-making.
  Conditional Flow Matching (CFM) critics are increasingly adopted in DRL, offering dense velocity supervision and a natural axis for scaling critic computation.
  Yet they carry a hidden \textbf{metric mismatch}: the distributional Bellman operator is a contraction in the Wasserstein distance, but CFM critics pair source and Bellman target samples arbitrarily, so their loss does not control the Wasserstein distance to the Bellman target.
  We show that this mismatch can be fixed with a single sorting operation. Our method, \textbf{FlowIQN}, sorts source and target samples within each mini-batch so that each source quantile is transported to its matching target quantile, approximating the monotone optimal transport coupling.
  We prove that this \textbf{quantile coupling} makes the CFM loss upper-bound, up to a constant, the squared 2-Wasserstein projection error. To our knowledge, this makes FlowIQN the first flow-matching critic whose training is an approximate Wasserstein projection of the Bellman target, and the guarantee extends to shortcut models for few-step inference.
  Empirically, the quantile coupling recovers multi-modal returns more accurately than an otherwise identical uncoupled flow with as few as two integration steps.
  FlowIQN improves substantially on the flow critic Value Flows in a fixed-policy diagnostic and is competitive on offline RL benchmarks across multiple policy extraction methods, identifying the coupling as a central design axis for flow-based critics.
  Code: \url{https://github.com/ori-goals/flowIQN}.

\end{abstract}

\vspace{-11pt}
\section{Introduction}
\label{sec:intro}
\aftersec

Standard \ac{RL} seeks to maximise the expected discounted return, whereas \ac{DRL} models the full random return distribution, retaining information that the expectation discards \cite{bdr2023}. This richer representation supports uncertainty-aware and risk-sensitive decision making \cite{pmlr-v80-dabney18a, dongzheng2026value, ma2025dsac, pmlr-v97-mavrin19a, ma2021conservative}, can yield stronger instance-dependent learning guarantees \cite{wang2023benefits, wang2024more}, and often improves performance and stability \cite{pmlr-v80-dabney18a, dabney2018distributional, ma2025dsac, wurman2022outracing}, which matters most where policy improvement relies on an accurate critic, as in offline \ac{RL} \cite{levine2020offline, fujimoto2019off, kumar2020conservative}.

Classic \ac{DRL} methods represent return distributions using categorical supports \cite{bellemare2017distributional} or quantile functions \cite{pmlr-v80-dabney18a, dabney2018distributional}. Flow matching offers an alternative parameterisation: a state-action-conditioned velocity field transports samples from a simple base distribution to a return distribution \cite{albergo2025stochastic, lipman2024flowmatching, liu2023flow, tong2024otcfm}. Flow-based critics are increasingly adopted in scalar and distributional value learning \cite{agrawalla2026floq, chen2025unleashing, dongzheng2026value, espinosa2025expressive, zhong2025flowcritic}, because the ODE formulation provides dense supervision along the transport path and a natural axis for scaling critic computation through additional integration steps \cite{agrawalla2026does, agrawalla2026floq}.

The standard \ac{DRL} update can be viewed as a Bellman--projection procedure: apply the distributional Bellman operator to form a target return distribution, then project this target back into the chosen critic class. 
The projection metric matters: under a fixed policy, the distributional Bellman operator is a $\gamma$-contraction in the $p$-Wasserstein metric \cite{bellemare2017distributional}, and the usual convergence story relies on projecting Bellman targets in a way that is compatible with a metric in which the operator contracts \cite{rowland2018analysis}. For a critic that transports mass in return space, as a flow critic does, that metric is naturally the Wasserstein distance.
This creates an obstacle for a \ac{CFM} distributional critic. Standard \ac{CFM} trains a velocity field with an $L^2$ regression loss over source-target pairs, typically using an independent coupling between base samples and target samples \cite{agrawalla2026floq, albergo2025stochastic, dongzheng2026value, lipman2024flowmatching, liu2023flow}. 
In scalar return space, this arbitrary coupling means that nearby source values can be paired with incompatible Bellman target returns, encouraging averaged transport directions and criss-crossing paths rather than a quantile-to-quantile map.
We refer to this misalignment as a \emph{metric mismatch}: the loss of an independently coupled \ac{CFM} critic does not control its Wasserstein error to the Bellman target. Our aim is not to replace categorical and quantile critics, whose projection steps are well understood \cite{dabney2018distributional, rowland2018analysis}, but to give flow critics the same footing: a training loss that controls the Wasserstein error to the Bellman target.

\begin{figure}[t]
    \centering
    \includegraphics[width=0.87\linewidth]{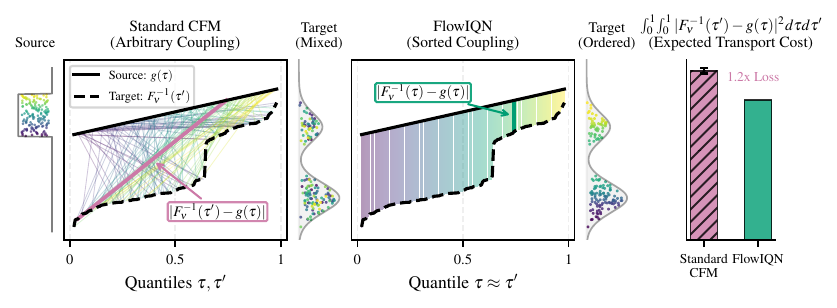}
    \vspace{-8pt}
    \caption{
    \textbf{FlowIQN aligns flow matching with Wasserstein geometry.} 
    \textbf{(Left)} Standard CFM arbitrarily couples source and target quantiles ($\tau, \tau^\prime$). This metric mismatch produces crossing paths and a double-integral expected cost that is not a Wasserstein projection. 
    \textbf{(Center \& Right)} FlowIQN's quantile coupling ($\tau \approx \tau^\prime$) approximates the 1D monotone optimal transport map. Because sorting pairs samples by their rank order, it explicitly transports each source quantile toward its exact corresponding target quantile. This monotone constraint ensures the regression objective upper-bounds, up to a constant, the 1D squared 2-Wasserstein distance, so that minimising the flow-matching loss is an approximate Wasserstein projection of the Bellman target.
}
    \label{fig:flowiqn_hero}
    \vspace{-16pt}
\end{figure}

We show that this requires only a single sorting operation. In one-dimensional return space, sorting source and Bellman target samples within each mini-batch and pairing them by rank, i.e.\ sorted mini-batch optimal transport \cite{tong2024otcfm}, approximates the monotone optimal transport map \cite{santambrogio2015optimal}, and makes the flow-matching loss control the Wasserstein projection error. We call this the \emph{quantile coupling}, and the resulting critic \textbf{FlowIQN}. It leaves the network, interpolant and regression loss unchanged. As an optional efficiency extension, we show that the coupling carries over to shortcut models \cite{frans2025one} for few-step inference. 
We study FlowIQN in controlled coupling ablations, fixed-policy evaluation and offline \ac{RL}. The ablations isolate the coupling, fixed-policy evaluation compares return-distribution accuracy with existing critics, and offline \ac{RL} tests whether the learned critic supports policy improvement. FlowIQN itself is a critic-learning method, and is not tied to any particular policy extraction rule.
Our contributions are as follows:
\begin{enumerate}
    \item We identify a metric mismatch: independently coupled \ac{CFM} critic losses do not control the Wasserstein distance to the Bellman target, a metric in which the distributional Bellman operator contracts (Section~\ref{sec:metric_mismatch}, Figure~\ref{fig:flowiqn_hero}).
    \item We propose FlowIQN, which replaces the independent coupling with a quantile coupling that approximates monotone optimal transport in 1D return space at the cost of a single sorting operation, and, as an optional efficiency extension, carry it over to shortcut models (Section~\ref{sec:method}).
    \item We define Wasserstein alignment as a property of a critic-training method and prove that FlowIQN satisfies it, so its critic update is an approximate Wasserstein projection compatible with contraction-based analyses of \ac{DRL}; to our knowledge, this is the first such result for a flow-matching distributional critic (Section~\ref{sec:theory}).
    \item We show empirically that the quantile coupling makes a flow critic more accurate than an otherwise identical uncoupled flow, which leads to better decisions when they depend on a mode or tail of the return distribution. Under a fixed policy, FlowIQN improves substantially on an existing flow critic and approaches classical distributional critics, and it is competitive on standard offline \ac{RL} benchmarks (Section~\ref{sec:experiments}).
\end{enumerate}

\beforesec
\section{Background and Preliminaries}
\aftersec

\subsection{Distributional Reinforcement Learning and Wasserstein Geometry}

As in standard \ac{RL}, we consider a Markov Decision Process (MDP) defined by the tuple
$M=(\mathcal{S}, \mathcal{A}, P, R, s_0, \gamma)$, where $\mathcal{S}$ and $\mathcal{A}$ denote the state and action spaces, $P(s' \mid s,a)$ is the transition function, $R(s,a)$ is the reward function, $s_0$ is the initial state distribution, and $\gamma \in (0,1)$ is the discount factor. Under a policy $\pi$, the discounted return distribution from state-action pair $(s,a)$ is the random variable
\begin{equation}
    Z^\pi(s,a) = \sum_{t=0}^{\infty} \gamma^t R(s_t,a_t),
\end{equation}
where trajectories evolve according to $a_t \sim \pi(\cdot|s_t)$ and $s_{t+1} \sim P(\cdot|s_t,a_t)$. In standard expected-return \ac{RL}, the corresponding action-value function is $Q^\pi(s,a) = \mathbb{E}[Z^\pi(s,a)]$.
In \ac{DRL}, the goal is to model the full return distribution $Z^\pi(s,a)$ rather than only its expectation. For a fixed policy $\pi$, the distributional Bellman operator $\mathcal{T}^\pi$ is defined by
\begin{equation}
    \mathcal{T}^\pi Z(s,a) \stackrel{D}{=} R(s,a) + \gamma Z(s',a'),
    \label{eq:dist_bellman}
\end{equation}
where $s' \sim P(\cdot \mid s,a)$, $a' \sim \pi(\cdot \mid s')$, and $\stackrel{D}{=}$ denotes equality in distribution. \ac{DRL} requires a notion of distance between probability distributions in order to study approximation and convergence.
A central result in \ac{DRL} is that, for a fixed policy $\pi$, the operator $\mathcal{T}^\pi$ is a $\gamma$-contraction under the supremal $p$-Wasserstein metric for $p \geq 1$ \cite{bellemare2017distributional}. This result provides a geometric foundation for distributional policy evaluation. 
Because returns are scalar, each return distribution is a 1D probability distribution. In this setting, the optimal coupling for the $p$-Wasserstein transport problem is given by monotone rearrangement, yielding the closed-form expression
\begin{equation}
    W_p(\mu,\nu)
    =
    \left(
        \int_0^1
        \left|
            F_\mu^{-1}(\tau) - F_\nu^{-1}(\tau)
        \right|^p
        \, d\tau
    \right)^{1/p},
    \label{eq:wasserstein_1d}
\end{equation}
where $F_\mu^{-1}$ and $F_\nu^{-1}$ denote the corresponding quantile functions \cite{santambrogio2015optimal}. Thus, for scalar return distributions, Wasserstein distance can be expressed directly as an $L^p$ error between quantile functions.
This contraction result motivates viewing \ac{DRL} updates as Bellman target construction followed by projection back into a restricted critic class. Different algorithms instantiate this projection step in different ways. C51 \cite{bellemare2017distributional} established the modern deep \ac{DRL} framework by using a categorical approximation over a fixed support; its projection is analysed in the Cram\'er distance rather than the Wasserstein distance \cite{rowland2018analysis}, a construction tied to the fixed support that does not transfer easily to flow critics. In contrast, quantile-based methods connect more directly to the 1D Wasserstein geometry: QR-DQN \cite{dabney2018distributional} represents the return distribution using a finite set of learned quantile locations and trains them with a quantile regression loss, while IQN \cite{pmlr-v80-dabney18a} extends this idea by learning an implicit quantile function. 

\beforesubsec
\subsection{Flow Matching for Return Distribution Modelling}
\aftersubsec

A flow-matching critic represents the return distribution through a learned transport process, mapping simple source noise samples to Bellman target return samples. In this formulation, the source-target coupling used during training directly shapes the critic objective.
Let $p_0$ denote a simple source distribution on $\mathbb{R}$. A flow-based critic parametrises a time-dependent velocity field $v_\theta(t,z_t|6s,a)$. For each state-action pair $(s,a)$ and initial sample $z_0 \sim p_0$, the velocity field defines the trajectory
$
    \frac{d z_t}{dt} = v_\theta(t,z_t|s,a).
$
Integrating this trajectory from $t=0$ to $t=1$ maps the initial sample $z_0$ to a terminal sample $z_1 = T_\theta(z_0 \mid s,a)$, inducing the model return distribution
$
    \eta_\theta(\cdot \mid s,a)
    =
    \bigl(T_\theta(\cdot \mid s,a)\bigr)_\# p_0,
$
where $(T_\theta)_\# p_0$ denotes the distribution obtained by applying $T_\theta$ to samples from $p_0$.
\ac{CFM} trains such a transport without differentiating through the ODE \cite{lipman2024flowmatching, lipman2024flowguide}. In \ac{DRL}, the relevant target at a state-action pair $(s,a)$ is the conditional return distribution. Let $Y(s,a) \stackrel{D}{=} R(s,a)+\gamma Z(s',a')$ denote the Bellman target random variable, and let $\nu(\cdot \mid s,a)$ denote its distribution. Given a source sample $z_0\sim p_0$, a target return sample $y\sim \nu(\cdot \mid s,a)$, and interpolation time $t\sim\mathrm{Unif}[0,1]$, \ac{CFM} forms the linear interpolant $z_t=(1-t)z_0+ty$ and regresses the critic velocity to the target displacement $y-z_0$:
\begin{equation}
    \mathcal{L}_{\mathrm{CFM}}^{\mathrm{ret}}(\theta; s,a)
    =
    \mathbb{E}_{t,\,z_0 \sim p_0,\, y \sim \nu(\cdot \mid s,a)}
    \left[
        \|v_\theta(t,z_t|s,a) - (y-z_0)\|_2^2
    \right].
    \label{eq:ret_cfm_loss}
\end{equation}

Averaging Eq.~\eqref{eq:ret_cfm_loss} over state-action pairs trains $\eta_\theta(\cdot \mid s,a)$ to match the Bellman target distribution $\nu(\cdot \mid s,a)$ by transporting source samples $z_0\sim p_0$ to target return samples $y\sim\nu(\cdot \mid s,a)$. Since returns are scalar, it is natural to instantiate the source as $z_0\sim\mathrm{Unif}[0,1]$, so that source values can be interpreted as quantile levels, similar to IQN \cite{pmlr-v80-dabney18a}. However, Eq.~\eqref{eq:ret_cfm_loss} couples $z_0$ and $y$ independently: the same source value can be paired with different target quantiles across updates, so the learned map is not forced to represent a monotone quantile transport. Thus, while flow matching provides a flexible sampled Bellman-target training objective, Eq.~\eqref{eq:ret_cfm_loss} alone does not control the Wasserstein distance to the Bellman target distribution. This is the key issue we address next.

\beforesubsec
\subsection{The Metric Mismatch}
\aftersubsec
\label{sec:metric_mismatch}

The previous subsection introduced a flow-matching objective for fitting Bellman target return distributions. We now ask what property this objective should satisfy to inherit the standard fixed-policy \ac{DRL} convergence argument. For a fixed policy $\pi$, distributional critic learning can be viewed as alternating between Bellman target construction and projection: first apply the distributional Bellman operator in Eq.~\eqref{eq:dist_bellman} to obtain the target law $\mathcal{T}^\pi Z(s,a)$, then fit a critic $\eta_\theta(\cdot\mid s,a)$ to that target. Since $\mathcal{T}^\pi$ is a $\gamma$-contraction in the supremal $p$-Wasserstein metric \cite{bellemare2017distributional}, this projection should approximately minimise Wasserstein distance to the Bellman target if the critic update is to remain aligned with the contraction geometry.

This requirement is where independently coupled flow matching becomes problematic. Eq.~\eqref{eq:ret_cfm_loss} trains the velocity field by drawing source samples $z_0\sim p_0$ and Bellman target samples $y\sim\nu(\cdot\mid s,a)$ independently, then regressing to the displacement $y-z_0$. Although this is a valid flow-matching objective for fitting the target distribution, it does not enforce any correspondence between source values and target quantile levels. In the one-dimensional return setting, however, Wasserstein distance is characterised by matching quantile functions, so a loss that controls the Wasserstein projection error should respect the monotone source-target coupling implied by Eq.~\eqref{eq:wasserstein_1d}. Establishing such a coupling, and showing that the resulting flow-matching objective upper bounds the Wasserstein projection error, is the goal of our method and theory.

Existing flow-based critics exploit flow matching for value learning, but do not explicitly enforce this requirement \cite{agrawalla2026floq, chen2025unleashing, dongzheng2026value, espinosa2025expressive, zhong2025flowcritic}. Instead, standard flow-matching and related score-based objectives are typically analysed through $L^2$, Fisher-divergence, or KL-like path-space viewpoints \cite{albergo2025stochastic, JMLR:v6:hyvarinen05a, lipman2024flowmatching, silveri2024theoretical}, rather than as Wasserstein projections. This leaves a \emph{metric mismatch}: for critics that transport mass in return space, Wasserstein distance is the natural geometry for the Bellman projection, yet independently coupled flow-matching losses do not control it. FlowIQN addresses this mismatch with a single sorting operation in Eq.~\eqref{eq:ret_cfm_loss}: it sorts the source and Bellman target samples and pairs them by rank, leaving the network, interpolant and regression loss unchanged.

\beforesec
\section{Method}
\label{sec:method}
\aftersec

We introduce \textbf{FlowIQN}: a flow-based distributional critic \textbf{whose projection step is aligned with the Wasserstein geometry underlying distributional Bellman contraction}. The only change from Eq.~\eqref{eq:ret_cfm_loss} is the coupling: FlowIQN replaces the independent source-target coupling of standard \ac{CFM} with a quantile coupling in 1D return space, leaving the network, interpolant and regression loss unchanged. We implement this coupling using sorted mini-batch optimal transport \cite{tong2024otcfm}, which approximates the monotone rearrangement of quantiles that defines optimal transport in one dimension \cite{santambrogio2015optimal}. We first define the critic objective (Section~\ref{sec:flowiqn}), then describe an optional shortcut variant for few-step inference (Section~\ref{sec:shortcut}).

\beforesubsec
\subsection{FlowIQN}
\aftersubsec
\label{sec:flowiqn}

Let $\tau \sim \mathrm{Unif}[0,1]$ denote a source quantile fraction and let $z_0=g(\tau)$ be the corresponding source value, where $g$ is either the identity map or an affine rescaling to a chosen source interval (as in \cite{agrawalla2026floq}). 
Unlike the generic flow critic in Section~\ref{sec:metric_mismatch}, FlowIQN keeps this quantile level as an explicit conditioning input throughout the flow. We therefore parameterise the velocity field as
$
    v_\theta(t,z_t|s,a,\tau),
$
a scalar velocity at flow time $t$ and current value $z_t$, conditioned on the state-action pair $(s,a)$ and the source quantile level $\tau$. The induced quantile indexed critic is computed by numerical integration of the velocity field, 
\begin{equation}
    z_{m+1}
    =
    z_m + \Delta t_m v_\theta(t_m,z_m|s,a,\tau),
    \qquad
    m=0,\dots,M-1,
\end{equation}
with $z_0=g(\tau)$. We denote the final output by
$
    Q_\theta^{\mathrm{flow}}(s,a, z_0,\tau) \coloneqq z_M,
$
which induces the model return distribution via the pushforward of $\tau\sim\mathrm{Unif}[0,1]$
\begin{equation}
    \eta_\theta(\cdot \mid s,a)
    =
    \bigl(Q_\theta^{\mathrm{flow}}(s,a,g(\cdot),\cdot)\bigr)_\# \mathrm{Unif}[0,1].
\end{equation}

For each sampled transition $(s_i,a_i,r_i,s'_i)$, we construct an empirical Bellman target distribution $\hat{\nu}_i$ using the target critic $Q_{\bar{\theta}}^{\mathrm{flow}}$. Concretely, we draw a next action $a'_i\sim\pi(\cdot\mid s'_i)$ according to the chosen policy extraction rule and evaluate the target critic at $K$ quantile fractions:
\begin{equation}
    y_{i,k}
    =
    r_i + \gamma m_i
    Q_{\bar{\theta}}^{\mathrm{flow}}(s'_i,a'_i,g(\tau'_{i,k}),\tau'_{i,k}),
    \qquad
    k=1,\dots,K,
\end{equation}
where $m_i$ is the continuation mask and $\tau'_{i,k}\sim\mathrm{Unif}[0,1]$. These samples define $\hat{\nu}_i$.
To train the critic, we reuse these fractions as the source quantile fractions, $\tau_{i,k}=\tau'_{i,k}$, and pair them with the $K$ target samples by sorting both sets:
\[
    \tilde{\tau}_{i,1} \le \cdots \le \tilde{\tau}_{i,K},
    \qquad
    \tilde{y}_{i,1} \le \cdots \le \tilde{y}_{i,K}.
\]
We then set $\tilde{z}_{0,i,k}=g(\tilde{\tau}_{i,k})$. This sorted pairing is a mini-batch approximation to the monotone optimal transport coupling between the source distribution and the Bellman target distribution \cite{tong2024otcfm}.
Given $t \sim \mathrm{Unif}[0,1]$, we form the coupled interpolant
$
    z_{i,k,t}
    =
    (1-t)\tilde{z}_{i,k,0} + t\tilde{y}_{i,k},
$
with target velocity
$
    u_{i,k}
    =
    \tilde{y}_{i,k} - \tilde{z}_{i,k,0}.
$
The FlowIQN critic is trained with
\begin{equation}
\label{eq:flowiqn_loss}
    \mathcal{L}_{\mathrm{FlowIQN}}(\theta)
    =
    \mathbb{E}_{(s_i,a_i,r_i,s'_i,m_i),\,k,\,t}
    \left[
        \left|
            v_\theta(t,z_{i,k,t},s_i,a_i,\tilde{\tau}_{i,k})
            -
            (\tilde{y}_{i,k} - \tilde{z}_{i,k,0})
        \right|^2
    \right],
\end{equation}
where the expectation is over transitions sampled from the dataset, quantile indices $k\in\{1,\dots,K\}$, and interpolation times $t\sim\mathrm{Unif}[0,1]$.
This objective has a simple interpretation. In independently coupled \ac{CFM}, source and target samples are paired without regard to their order: the same source quantile level can be paired with a low Bellman target return in one update and a high target return in another. The velocity field is therefore not given a consistent target location for each source quantile. FlowIQN removes this ambiguity by sorting both sets of samples and pairing corresponding order statistics. The $k$-th smallest source value is paired with the $k$-th smallest Bellman target return, so lower source quantiles are consistently transported toward lower target returns, and higher source quantiles toward higher target returns (Figure~\ref{fig:flowiqn_hero}). The resulting velocity regression learns an ordered transport from source quantiles to Bellman target quantiles. In the 1D return setting, this is the monotone coupling underlying the Wasserstein projection view developed in Section~\ref{sec:theory}.

\beforesubsec

\subsection{Optional: Few-Step Inference via Shortcut Models}
\aftersubsec
\label{sec:shortcut}

This subsection describes an efficiency extension and can be skipped on a first reading: the quantile coupling and its guarantee in Section~\ref{sec:wa_projection} do not depend on it.
Base FlowIQN integrates the velocity field with $M$ Euler steps, and using fewer steps increases discretisation error. Following shortcut models \cite{frans2025one}, we additionally condition the network on a step size $d$ and train it to predict the average velocity $s_\theta(t,z_t,d\mid s,a,\tau)$ over $[t,t+d]$, so that $z_{t+d}=z_t+d\,s_\theta(t,z_t,d\mid s,a,\tau)$ and a single step with $d=1$ maps a source sample directly to a return sample. The shortcut critic is trained with
\begin{equation}
    \mathcal{L}_{\mathrm{Shortcut}}(\theta)
    =
    (1-\lambda_{\mathrm{c}})\mathcal{L}_{\mathrm{FlowIQN}}(\theta)
    + \lambda_{\mathrm{c}}\mathcal{L}_{\mathrm{con}}(\theta),
\end{equation}
where $\mathcal{L}_{\mathrm{FlowIQN}}$ anchors the critic to the quantile-coupled Bellman targets and $\mathcal{L}_{\mathrm{con}}$ is a self-consistency loss requiring one step of size $2d$ to agree with two composed steps of size $d$ (Appendix~\ref{appendix:shortcut_background}). Because the quantile-coupled target paths are straight, the shortcut changes the cost of inference but not the projection target (Section~\ref{sec:shortcut_theory}).

\beforesec
\section{Theoretical Analysis}
\aftersec
\label{sec:theory}

We now show that, in the 1D return setting, FlowIQN resolves the metric mismatch identified in Section~\ref{sec:metric_mismatch}. We first make precise what it means for a critic objective to be aligned with the Wasserstein geometry in which the distributional Bellman operator is contractive, and then show that the quantile-coupled objective has this property. Full proofs are deferred to Appendix~\ref{appendix:theory_and_proofs}.

\beforesubsec
\subsection{FlowIQN as a Wasserstein-aligned projection}\label{sec:wa_projection}
\aftersubsec

For a state-action pair $(s,a)$, let $\nu(\cdot\mid s,a)$ denote the Bellman target distribution and $\eta_\theta(\cdot\mid s,a)$ the distribution induced by the critic. In 1D, the optimal transport coupling is given by monotone rearrangement, so quantile-coupled flow matching trains the critic along the straight-line paths $z_t^\star=(1-t)g(\tau)+tF_\nu^{-1}(\tau\mid s,a)$ between each source quantile and its target quantile.

\begin{definition}[Wasserstein alignment]
\label{def:wasserstein_aligned}
Let $\mathcal M$ be a critic-training method with loss $\mathcal L_{\mathcal M}(\theta;s,a)\ge 0$ over a model class $\Theta$. We say that $\mathcal M$ is \emph{Wasserstein-aligned} if (i) there is a finite constant $C_{\mathcal M}$, fixed by the method (its coupling, integrator and model class) and independent of $\theta$, such that
\[
W_2^2\!\bigl(\eta_\theta(\cdot\mid s,a),\nu(\cdot\mid s,a)\bigr)\;\le\;C_{\mathcal M}\,\mathcal L_{\mathcal M}(\theta;s,a)
\qquad\text{for all }\theta\in\Theta,
\]
and (ii) $\mathcal L_{\mathcal M}$ has no irreducible component, i.e.\ $\mathcal L_{\mathcal M}(\theta;s,a)=0$ is attained by a critic with $\eta_\theta(\cdot\mid s,a)=\nu(\cdot\mid s,a)$.
\end{definition}

Condition (i) means that minimising the critic loss controls the Wasserstein projection error to the Bellman target. Condition (ii) requires that the loss can reach zero when the critic matches the target, so that (i) cannot be satisfied by a loss that always stays positive.

\begin{proposition}[Conditional Wasserstein upper bound]
\label{prop:w2_upper_bound}
Fix $(s,a)$ and suppose the FlowIQN velocity field is $L$-Lipschitz in the flow state $z$, i.e.\ $|v_\theta(t,z\mid s,a,\tau)-v_\theta(t,z'\mid s,a,\tau)|\le L|z-z'|$ for all $z,z'\in\mathbb R$ and $t,\tau\in[0,1]$. Then
\[
W_2^2\!\bigl(\eta_\theta(\cdot\mid s,a), \nu(\cdot\mid s,a)\bigr)
\;\le\;
C_L\,\mathcal{L}_{\mathrm{FlowIQN}}(\theta; s,a),
\qquad
C_L=\frac{e^{2L}-1}{2L}\ge 1,
\]
where $\mathcal{L}_{\mathrm{FlowIQN}}(\theta; s,a)$ is the population counterpart of Eq.~\eqref{eq:flowiqn_loss} at $(s,a)$ (proof in Appendix~\ref{app:proof_prop1}).
\end{proposition}

Proposition~\ref{prop:w2_upper_bound} states that the FlowIQN objective controls the Wasserstein projection error of the critic. The loss is evaluated on the straight target paths, whereas the critic is integrated along its own trajectory; the Lipschitz constant converts the former into the latter via Gr\"onwall's inequality. For a model class whose velocity fields are uniformly $\bar L$-Lipschitz in $z$ (e.g.\ MLPs with bounded weight norms), FlowIQN is therefore Wasserstein-aligned with $C_{\mathcal M}=C_{\bar L}$. Condition (ii) holds because the monotone coupling pairs each $\tau$ with a single target $F_\nu^{-1}(\tau\mid s,a)$, whereas independently coupled \ac{CFM} regresses onto displacements that remain random given $(t,z_t)$, so its loss stays strictly positive even when $\eta_\theta=\nu$ (Appendix~\ref{app:wa_discussion}). Minimising $\mathcal{L}_{\mathrm{FlowIQN}}$ is thus an approximate projection of the Bellman target back into the critic class under the Wasserstein metric. This achieves the aim set out in Section~\ref{sec:intro}: with the quantile coupling, a flow-matching critic's loss controls its Wasserstein error to the Bellman target, resolving the \emph{metric mismatch} of Section~\ref{sec:metric_mismatch}.
This interpretation places FlowIQN within the standard approximate-projection framework of distributional \ac{RL}.

\begin{corollary}[Approximate-projection recursion]
\label{cor:approx_proj}
Let $\bar d_p$ denote the supremal $p$-Wasserstein metric over state-action pairs, and let $Z^\pi$ denote the fixed point of the distributional Bellman operator $\mathcal T^\pi$.
Suppose the critic update at iteration $k$ yields a distribution $Z_{k+1}$ satisfying
\[
\bar d_p\!\bigl(Z_{k+1}, \mathcal T^\pi Z_k\bigr)\le \varepsilon_k.
\]
Then
\[
\bar d_p(Z_{k+1}, Z^\pi)
\;\le\;
\gamma\,\bar d_p(Z_k, Z^\pi) + \varepsilon_k.
\]
(Proof in Appendix~\ref{app:proof_cor1}.)
\end{corollary}

Corollary~\ref{cor:approx_proj} shows that once the critic update can be interpreted as an approximate Wasserstein projection, the standard contraction-based analysis of \ac{DRL} applies directly: the Bellman operator contributes the contraction factor $\gamma$, while approximation error enters only through $\varepsilon_k$. Proposition~\ref{prop:w2_upper_bound} provides this missing link for FlowIQN: for $p=2$ it gives $\varepsilon_k=\sup_{s,a}\bigl(C_L\,\mathcal L_{\mathrm{FlowIQN}}(\theta_{k+1};s,a)\bigr)^{1/2}$, with the loss taken against the target $\mathcal T^\pi Z_k$, ensuring that the projection step is aligned with the Wasserstein geometry in which the distributional Bellman operator is well behaved.

\beforesubsec
\subsection{Shortcut models preserve the same target solution}\label{sec:shortcut_theory}
\aftersubsec

We next consider the shortcut variant introduced in Section~\ref{sec:shortcut}. The additional self-consistency loss improves inference efficiency, but it also introduces a potential concern: if the consistency objective changes the optimal critic solution, it could undermine the Wasserstein alignment established above. Under quantile coupling, however, the target transport paths are straight lines with constant velocity, so this does not occur.

\begin{proposition}[Zero shortcut bias under monotone coupling]
\label{prop:shortcut_bias}
Under monotone quantile coupling, the target transport path is linear in time and has constant velocity. Consequently, the optimal transport solution exactly satisfies the shortcut self-consistency relation, so adding the consistency loss does not change the target solution of the critic update (proof in Appendix~\ref{app:proof_prop2}).
\end{proposition}

\begin{corollary}[Single-step shortcut objective]
\label{cor:shortcut_w2}
For $d=1$, the one-step shortcut prediction maps each source quantile directly
to its coupled target quantile. The resulting single-step shortcut loss is exactly the squared 2-Wasserstein distance between the shortcut critic distribution and the Bellman target distribution in the 1D return setting whenever the one-step map is monotone in $\tau$, and an upper bound on it otherwise (proof in Appendix~\ref{app:proof_cor2}).
\end{corollary}

Unlike Proposition~\ref{prop:w2_upper_bound}, Corollary~\ref{cor:shortcut_w2} requires no Lipschitz assumption: the single-step map evaluates the velocity field only at $t=0$, where the model and target paths coincide, so the single-step shortcut critic is Wasserstein-aligned with $C_{\mathcal M}=1$.
Therefore, the shortcut parameterisation preserves the same target as the base FlowIQN critic, and hence its Wasserstein alignment, while enabling efficient few-step or single-step inference.

\beforesec
\section{FlowIQN in Offline RL}
\aftersec
\label{sec:flowiqn_offline_rl}

We use offline \ac{RL} as a downstream testbed because critic errors directly affect policy improvement, while the static dataset setting makes behaviour regularisation and policy extraction explicit. In the offline setting, we assume a static dataset
\(
\mathcal D=\{(s,a,r,s',m)\}
\)
and sample minibatches
\(
B=\{(s_i,a_i,r_i,s_i',m_i)\}_{i=1}^N
\)
for training. Critic updates use Bellman bootstrapping with target networks \cite{mnih2015human}, while policy improvement must remain close to the support of the dataset. 
Critic targets are constructed as in Section~\ref{sec:flowiqn}, with the next action chosen by the policy extraction rule; Appendix~\ref{appendix:offline_bellman_targets} gives the offline details, including the choice of source interval $g$.
FlowIQN is a critic-learning method rather than a policy-extraction method, so we evaluate two instantiations. First, we use an FQL-style actor (FlowIQN-FQL), where a Behaviour-Cloned (BC) flow policy models the dataset action distribution and a separate one-step actor is trained to maximize the expected value of the FlowIQN critic distribution under behavioural regularisation toward that flow policy. In this variant only, following \cite{agrawalla2026floq}, we distill the multi-step FlowIQN critic to a fast one-step student critic for actor updates. Second, we use rejection sampling (FlowIQN-R), where a BC flow policy proposes candidate actions and the action with the highest critic score is selected.

\beforesec
\section{Experiments}
\label{sec:experiments}
\aftersec

Our experiments address three questions, one per subsection:
\textbf{(i)} Does the quantile coupling make a flow critic more accurate than an otherwise identical uncoupled flow, and when does this change decisions?
\textbf{(ii)} How accurate is FlowIQN relative to existing distributional critics under a fixed policy?
\textbf{(iii)} Is FlowIQN competitive in offline \ac{RL} across policy extraction rules?
We also study shortcut inference and adaptive integration in Appendix~\ref{appendix:results}.

\beforesubsec
\subsection{Isolating the quantile coupling}
\aftersubsec
\label{sec:coupling_ablation}

\begin{table}[t]
\caption{\textbf{Isolating the quantile coupling.} FlowIQN and the uncoupled flow differ only in the sorting operation; flow critics are evaluated at 2 and 32 Euler steps. Means over 5 seeds (6 for risk); full results in Appendix~\ref{app:coupling_ablation}.}
\label{tab:coupling_ablation}
\centering
\small
\setlength{\tabcolsep}{4pt}
\begin{tabular}{llcccccc}
\toprule
& & \multicolumn{2}{c}{FlowIQN} & \multicolumn{2}{c}{Uncoupled} & & \\
\cmidrule(lr){3-4}\cmidrule(lr){5-6}
Task & Metric & 2 & 32 & 2 & 32 & IQN & C51 \\
\midrule
Accuracy, bimodal & $W_2$ $\downarrow$ & $0.73$ & $0.65$ & $1.87$ & $1.50$ & $0.17$ & $0.23$ \\
Accuracy, trimodal & $W_2$ $\downarrow$ & $0.59$ & $0.65$ & $1.73$ & $1.09$ & $0.13$ & $0.15$ \\
Accuracy, quadmodal & $W_2$ $\downarrow$ & $0.48$ & $0.54$ & $1.83$ & $1.23$ & $0.09$ & $0.09$ \\
Decision ($q_{0.9}$ rule) & optimal action $\uparrow$ & $1.00$ & $1.00$ & $0.12$ & $0.57$ & $1.00$ & $1.00$ \\
Risk (CVaR$_{0.1}$) & CVaR$_{0.1}$ $\uparrow$ & $-18.55$ & $-18.55$ & $-20.08$ & $-20.08$ & $-18.50$ & $-18.50$ \\
\bottomrule
\end{tabular}
\vspace{-8pt}
\end{table}

\textbf{Setup.}
To isolate the coupling, we toggle the sorting operation in FlowIQN while using an identical network, initialisation, data, interpolant and compute. Without the sort, source and target samples are paired independently. Targets are re-sampled Monte Carlo draws, and we vary the number of Euler steps from 2 to 32. We consider three tasks (details in Appendix~\ref{app:coupling_ablation}): \emph{critic accuracy} on multi-modal returns with closed-form ground truth; a \emph{decision} task whose greedy action depends on the shape of the return, under a risk-seeking rule of picking the action with the better 90th-percentile outcome; and \emph{risk-sensitive control} in the Machine-Replacement MDP under CVaR$_{0.1}$, as used by Value Flows \cite{dongzheng2026value}. IQN and C51 are trained on the same targets as references.

\textbf{Results.}
Table~\ref{tab:coupling_ablation} answers question~\textbf{(i)}. FlowIQN recovers multi-modal returns far more accurately than the identical uncoupled flow at a fixed inference budget (${\approx}2$-$4\times$ lower $W_2$), preserving the separation between modes over which the uncoupled flow averages (Figure~\ref{fig:density_overlay}). Already at 2 Euler steps this narrows the gap to IQN and C51, albeit without surpassing them. In the decision task, the choice is between a risky action with a high-return mode and a safe action with a higher mean, so a critic that averages over the modes selects the safe action. At every budget, FlowIQN selects the optimal action in every state, matching IQN and C51, whereas the uncoupled flow is correct 12\% of the time at 2 steps, rising to 57\% at 32. In the risk-sensitive task, FlowIQN reaches the CVaR$_{0.1}$-optimal policy at every budget, down to 2 Euler steps (5 of 6 seeds), matching IQN, C51 and Value Flows, whereas the uncoupled flow does not reach it at any budget and instead converges to the risk-neutral policy. The sorting operation is what makes the flow's $\tau$-index a genuine quantile index, and therefore CVaR$_{0.1}$ reads the lower tail directly. Without it, the read is wrong at any number of Euler steps.

\textbf{When the coupling helps.}
The uncoupled flow does not necessarily lead to a poor-performing policy. In a companion task in which the two actions have equal expected returns and differ only in variance, mode averaging leaves the ordering of their $0.9$-quantiles unchanged, so every critic, including the uncoupled flow at 2 Euler steps, selects the optimal action (Appendix~\ref{app:coupling_ablation}). The sorting operation changes the selected action when the decision requires resolving a high-return mode or a tail, and more Euler steps narrow the gap without closing it. Figure~\ref{fig:w2_performance_profiles} is instead a system-level comparison, so FlowIQN's gain over Value Flows reflects the full critic rather than the coupling alone.

\beforesubsec
\subsection{Fixed-policy return distribution modelling}
\aftersubsec
\label{sec:policy_eval_exp}
\begin{wrapfigure}[25]{r}{0.59\columnwidth}
  \vspace{-14pt}
  \centering
  \includegraphics[width=\linewidth]{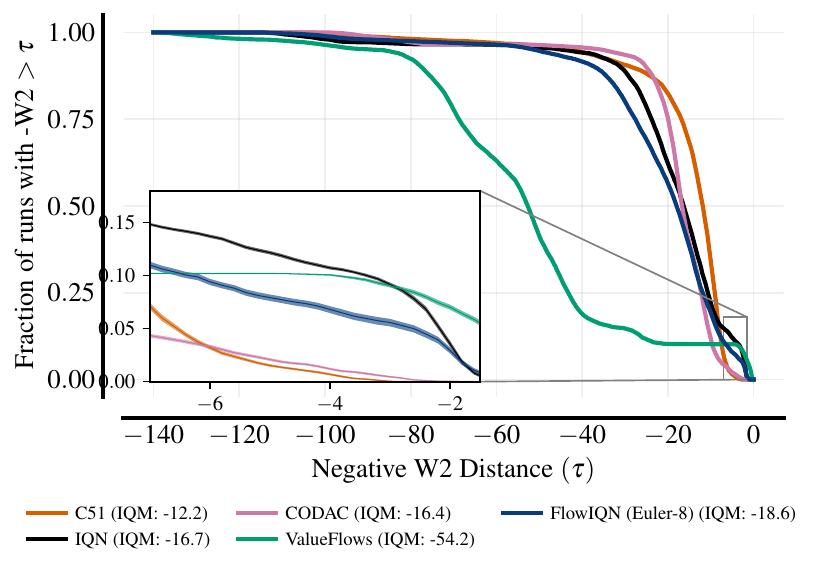}
  \caption{
\textbf{FlowIQN compared with existing distributional critics under a fixed policy.}
Performance profiles of negative empirical $W_2$ distance to Monte Carlo (MC) return targets; higher is better. FlowIQN improves substantially on the prior flow critic Value Flows and approaches IQN and CODAC; C51 remains the most accurate. Shaded regions are $95\%$ stratified-bootstrap confidence intervals over seeds, magnified in the inset; per-seed IQMs are in Appendix~\ref{app:per_seed}.
}
  \vspace{-12pt}
\label{fig:w2_performance_profiles}
\end{wrapfigure}

Offline RL performance depends on both critic learning and policy improvement, making it difficult to isolate whether the critic accurately models return distributions. Therefore we evaluate critics under a fixed policy. By removing policy improvement, critics are trained solely to approximate the return distribution $Z^\pi(s,a)$ induced by a fixed policy $\pi$, allowing us to assess distributional accuracy in isolation.

\textbf{Setup.}
We train an FQL~\citep{fql_park2025} policy $\pi_{\mathrm{FQL}}$ and train all critics using Bellman targets induced by this fixed policy. 
Prior evaluations of flow-based return critics can conflate critic accuracy with policy performance, since critics trained under different improving policies are compared to MC returns from a separate reference policy~\citep{dongzheng2026value}. We avoid this by freezing $\pi_{\mathrm{FQL}}$ and evaluating all critics against return distributions induced by this same fixed policy. We evaluate 5 critic seeds on 1578 state-action pairs from \texttt{scene-play-singletask-task2-v0}, using MC rollouts of $\pi_{\mathrm{FQL}}$ to estimate target return distributions (details in Appendix~\ref{appendix:w2_policy_eval}).

\textbf{Baselines and metric.}
We compare against the critics from \textsc{C51}, \textsc{IQN}, \textsc{CODAC}, and \textsc{Value Flows}, spanning categorical, quantile-based, and flow-based distributional approaches. We report empirical negative $2$-Wasserstein distance to MC return targets and the corresponding inter-quartile mean (IQM).

\vspace{-4pt}
\paragraph{Results.}
Figure~\ref{fig:w2_performance_profiles} answers question~\textbf{(ii)}. FlowIQN (IQM $-18.6$) is substantially more accurate than the prior flow critic Value Flows ($-54.2$) and approaches IQN ($-16.7$) and CODAC ($-16.4$), while C51 ($-12.2$) remains the most accurate. FlowIQN's accuracy varies more across seeds than that of the other critics, with per-seed IQMs from $-16.9$ to $-20.0$ (Appendix~\ref{app:per_seed}), and is insensitive to $K$ (Appendix~\ref{app:k_sweep}). FlowIQN therefore makes flow critics competitive with classical distributional critics on this diagnostic.

\beforesubsec
\subsection{Main Offline RL Results}
\aftersubsec

\begin{wrapfigure}{r}{0.4\columnwidth}
  \vspace{-25pt}
  \centering
  \includegraphics[width=0.4\columnwidth]{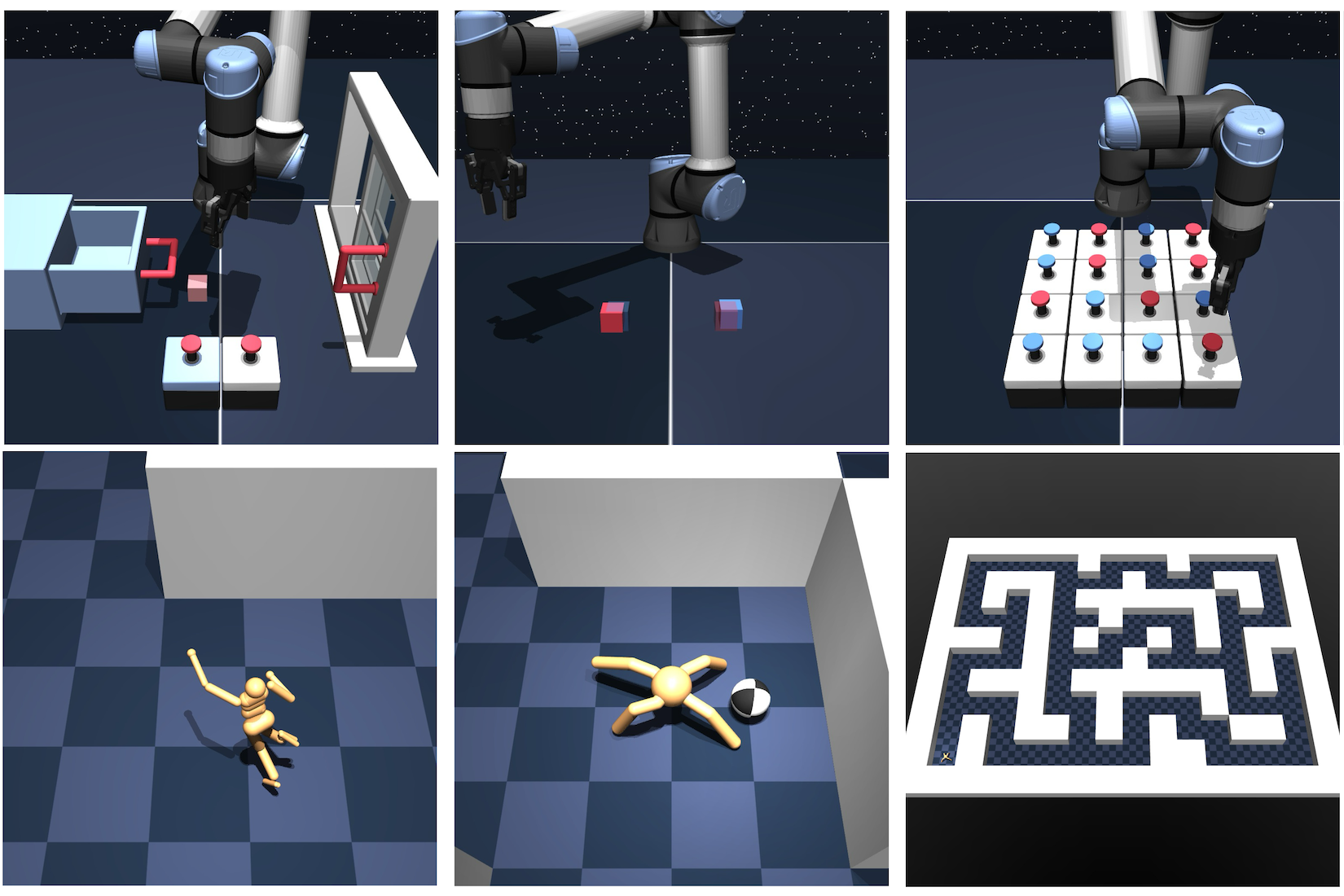}
  \caption{\textbf{OGBench \cite{ogbench_park2025} tasks}, including manipulation and locomotion tasks.}
  \vspace{-14pt}
  \label{fig:ogbench_envs}
\end{wrapfigure}

\textbf{Offline \ac{RL} tasks and datasets.}
Following recent works in offline \ac{RL} \cite{agrawalla2026floq, dongzheng2026value, fql_park2025}, we use the reward-based single-task variants of OGBench tasks as our primary benchmark \cite{ogbench_park2025}. 
OGBench provides a range of long-horizon robotics tasks, with sparse or semi-sparse rewards, and highly multimodal action distributions, making it a strong testbed for critic quality in offline \ac{RL}. 
We evaluate on $8$ state-based OGBench environments with $5$ tasks per environment, yielding $40$ tasks in total.

\textbf{Baselines.}
We include: \textsc{BC} and \textsc{ReBRAC} as strong scalar-critic baselines; \textsc{IQN} and, where available, \textsc{CODAC}, as non-flow distributional critics; \textsc{FQL} to isolate the effect of replacing a scalar critic within a flow-policy actor-extraction pipeline; and \textsc{floq} and \textsc{Value Flows} as recent flow-based critic baselines. Value Flows includes additional variance weighting and regularisation, so we treat it as a strong system-level baseline rather than a pure critic-parametrisation ablation. For more detailed discussion, see Appendix~\ref{appendix:offline_rl_baselines}

\textbf{Policy extraction.}
To highlight FlowIQN as a critic-learning contribution, we evaluate it with both policy extraction rules described in Section~\ref{sec:flowiqn_offline_rl}: an FQL-style one-step actor, and rejection sampling over a \textsc{BC} model. This lets us test whether any gains from FlowIQN persist across different control-layer choices, rather than depending on a single actor design. Prior work already uses these two strategies as viable extraction rules for expressive offline \ac{RL} systems \cite{fql_park2025, dongzheng2026value}. 

\textbf{Evaluation protocol and reported baselines.}
We follow the standard OGBench/FQL evaluation convention \cite{ogbench_park2025, fql_park2025} used by recent work on this benchmark \cite{dongzheng2026value, agrawalla2026floq}. 
We rerun IQN and Value Flows using the released Value Flows implementation and hyperparameters\cite{dongzheng2026value}. For remaining baselines, when results are available under the same benchmark/task definitions and evaluation protocol, we report source-paper numbers and mark them explicitly in all tables. For our own runs, we report fixed-horizon offline performance, taking the average of the last three evaluation epochs as in \cite{ogbench_park2025}, rather than best-over-training checkpoints. Aggregated task performance is presented in Table \ref{tab:aggregated_results_v2}.
Full architectural, optimization, seed, and hyperparameter details are deferred to Appendix~\ref{appendix:experimental_setup}.

\textbf{Results.}
Table~\ref{tab:aggregated_results_v2} answers question~\textbf{(iii)}. FlowIQN is competitive in offline \ac{RL} with both policy extraction rules. FlowIQN-R achieves the strongest result on \texttt{cube-double} and \texttt{puzzle-4x4}, while \textsc{IQN} is strongest on \texttt{cube-triple}, Value Flows on \texttt{puzzle-3x3}, and floq on the long-horizon navigation tasks. Under the matched FQL-style extraction, FlowIQN-FQL is competitive with \textsc{IQN} and \textsc{CODAC} without consistently outperforming them. We therefore view these results as evidence that quantile-coupled flow critics are a competitive and sometimes beneficial critic-learning mechanism, rather than as evidence that they uniformly dominate existing offline \ac{RL} methods.

Fixed-policy accuracy is therefore not a complete predictor of downstream performance, which also depends on actor extraction, behaviour regularisation, action support and target-network dynamics. This is consistent with Value Flows remaining a strong offline baseline despite its fixed-policy $W_2$, and with recent work arguing that flow critics aid TD learning through iterative computation, dense velocity supervision and improved plasticity~\citep{agrawalla2026does}. We therefore view FlowIQN as a Wasserstein-aligned critic objective that can serve as a theoretically grounded building block within such systems, rather than as a universally superior offline \ac{RL} method.

\begin{table*}[t]
\vspace{-5pt}
\caption{Aggregated offline RL results across OGBench tasks. We report average task success $\pm$ standard deviation, aggregated across all 5 subtasks per environment. We denote values at or above $95\%$ the best performance in bold, following \cite{ogbench_park2025, fql_park2025}. This highlighting is intended as a descriptive convention rather than indicating statistical significance.}
\vspace{-3pt}
\begin{center}
\label{tab:aggregated_results_v2}
\scriptsize 
\setlength{\tabcolsep}{3pt}
\renewcommand{\arraystretch}{0.95}

\begin{adjustbox}{max width=\textwidth}
\begin{tabular}{l cc ccccccc} %
\toprule
 & \multicolumn{2}{c}{\textbf{Gaussian Policies}} & \multicolumn{7}{c}{\textbf{Flow Policies}} \\
\cmidrule(lr){2-3} \cmidrule(lr){4-10} %
\textbf{Environment} & BC & ReBRAC & IQN & CODAC & FQL & floq & Value Flows & FlowIQN-FQL & FlowIQN-R \\
\midrule

cube-double & $2\pm1$  & $12\pm1$  & $64\pm3$           & $61\pm6$  & $25\pm6$  & $47\pm15$ & $63\pm4$ & $52\pm9$ & $\mathbf{72\pm6 }$ \\
cube-triple & $0\pm0$  & $0\pm0$   & $\mathbf{24\pm2}$            & $2\pm1$   & $4\pm2$   & -         & $20\pm6$ & $10\pm4$ & $5\pm4$  \\
puzzle-3x3  & $2\pm0$  & $21\pm1$  & $31\pm1$           & $20\pm5$  & $29\pm5$  & $37\pm7$  & $\mathbf{41\pm18}$& $32\pm3$ & $30\pm4$  \\
puzzle-4x4  & $0\pm0$  & $14\pm1$  & $29\pm2$           & $20\pm18$ & $9\pm3$   & $28\pm6$ & $24\pm4$ & $17\pm2$ & $\mathbf{47\pm10}$\\
scene       & $5\pm1$  & $41\pm3$  & $\mathbf{59\pm1}$           & $55\pm1$  & $\mathbf{57\pm4}$ & $\mathbf{57\pm2}$ & $\mathbf{57\pm2}$ & $56\pm1$ & $\mathbf{59\pm1}$ \\
antmaze-giant & $0\pm0$ & $26\pm8$ & - & - & $27\pm23$ & $\mathbf{51\pm12}$ & - & $22\pm11$ & $25\pm4$ \\
hmmaze-large & $1\pm0$ & $2\pm1$ & - & - & $16\pm9$ & $\mathbf{28\pm9}$ & - & $17\pm3$ & $19\pm4$ \\
antsoccer-arena  & $1\pm0$ & $0\pm0$ & - & - & $61\pm10$ & $\mathbf{65\pm12}$ & - & $44\pm3$  & $61\pm2$ \\

\midrule
\end{tabular}
\end{adjustbox}
\end{center}
\begin{minipage}{\textwidth}
\footnotesize
\textbf{Note.} BC, FQL, and floq results are re-reported from \cite{agrawalla2026floq}. CODAC results are re-reported from \cite{dongzheng2026value}. ReBRAC uses the available reported results from \cite{agrawalla2026floq,dongzheng2026value}. IQN and Value Flows are rerun using the published code and hyperparameters from \cite{dongzheng2026value}. Dashes indicate unavailable results: CODAC and floq are not reported on these environments in their source papers, and we re-ran IQN and Value Flows on the manipulation environments only (Appendix~\ref{app:missing_cells}).
\end{minipage}
\vspace{-15pt}
\end{table*}

\beforesec
\section{Limitations}
\aftersec

\label{sec:limitations}
As for existing distributional critics such as C51, QR-DQN and IQN~\citep{bellemare2017distributional, dabney2018distributional, pmlr-v80-dabney18a}, our theoretical analysis covers fixed-policy distributional Bellman evaluation. In this setting, the Wasserstein contraction structure is cleanly characterized, but the distributional optimality operator is not a contraction in general~\citep{bellemare2017distributional}. In full offline control, the critic co-evolves with policy extraction, behaviour regularisation, action support, and target-network dynamics. Hence, fixed-policy $W_2$ accuracy should be treated as a diagnostic rather than a complete predictor of downstream performance.
FlowIQN also relies on a finite-sample approximation to monotone optimal transport: we sort only $K$ bootstrapped target samples per transition, so the empirical coupling is a noisy estimate of the population one, although accuracy on our fixed-policy diagnostic is insensitive to $K$ between 1 and 32 (Appendix~\ref{app:k_sweep}). Larger batches with architectures that exploit batch-size scaling~\citep{wang2026} may reduce this noise.

Proposition~\ref{prop:w2_upper_bound} is stated for the population loss (Eq.~\eqref{eq:flowiqn_population_loss}), which pairs each $\tau$ with the target quantile $F_\nu^{-1}(\tau\mid s,a)$, and the sort estimates this pairing from samples of $\nu(\cdot\mid s,a)$. When targets can be re-sampled, as in Section~\ref{sec:coupling_ablation}, these are samples of $\nu$ itself. In offline TD learning, however, each transition provides a single reward and successor state, so with stochastic rewards or transitions the sort sees only the target distribution given that successor. The critic then learns the average of the per-successor quantile functions, which under-estimates the spread of the return distribution, and increasing $K$ does not change this. The effect does not arise under deterministic dynamics, as in our offline benchmarks; otherwise, several successors per state-action pair could be sampled.

Finally, FlowIQN models scalar returns, as do the other distributional critics discussed in this section. For returns in $\mathbb R^d$ with $d=1$, a sort computes the optimal coupling exactly; for $d\ge2$ there is no such ordering, so a minibatch optimal-transport solver~\citep{tong2024otcfm} or iterative reflow~\citep{liu2023flow} could replace it, with Proposition~\ref{prop:w2_upper_bound} re-derived for the new coupling.

\beforesec
\section{Conclusion}
\aftersec

We introduce FlowIQN, a quantile-coupled flow-matching critic that aligns flow-based return distribution modelling with the Wasserstein geometry of distributional reinforcement learning. FlowIQN replaces the independent source-target coupling of standard \ac{CFM} with sorted mini-batch optimal transport in the 1D return setting, a single sorting operation that yields a Wasserstein-aligned approximate projection to the Bellman target and extends naturally to shortcut models for efficient inference. 
Empirically, the coupling makes a flow critic substantially more accurate than an otherwise identical uncoupled flow.
FlowIQN improves fixed-policy Wasserstein return-distribution accuracy over Value Flows, and it remains competitive in offline \ac{RL}. More broadly, our results identify coupling as a central design axis for flow-based critics: choosing a coupling matched to the Bellman contraction geometry can improve return-distribution modelling while providing a reusable foundation for future offline \ac{RL} systems.

\begin{ack}

M.G. was supported by an EPSRC iCASE studentship (EP/Y52878X/1, iCASE Voucher 230058) in partnership with Dyson Technology Ltd. V.-A.D and N.H. were supported by the Engineering and Physical Sciences Research Council (EPSRC) From Sensing to Collaboration Programme Grant (EP/V000748/1). This research was supported with Cloud TPUs from Google's TPU Research Cloud (TRC). The authors would like to acknowledge the use of the University of Oxford Advanced Research Computing (ARC) facility in carrying out this work (\url{https://doi.org/10.5281/zenodo.22558}). The authors declare that they have no competing interests relating to this submission.
\end{ack}

\bibliographystyle{abbrvnat}
\bibliography{references}

\appendix

\section{Additional Background and Preliminaries}
\label{appendix:background}

This appendix collects supplementary definitions used throughout the paper. The main text contains the background needed to follow the method; here we provide additional detail on the probability-metric, flow-matching, and shortcut-model viewpoints used in the theoretical analysis.

\subsection{Wasserstein distances and 1D quantile maps}
\label{appendix:wasserstein_background}

For probability measures $\mu$ and $\nu$ on a metric space $\mathcal X$ with finite $p$-th moments, the $p$-Wasserstein distance is
\begin{equation}
    W_p(\mu,\nu)
    =
    \left(
        \inf_{\lambda\in\Pi(\mu,\nu)}
        \int_{\mathcal X\times \mathcal X}
        d(x,y)^p\,d\lambda(x,y)
    \right)^{1/p},
\end{equation}
where $\Pi(\mu,\nu)$ denotes the set of couplings with marginals $\mu$ and $\nu$ \cite{santambrogio2015optimal}. Unlike density-based divergences, Wasserstein distances depend on the geometry of the underlying space through the ground metric $d$. In this work, returns are scalar, so $\mathcal X=\mathbb R$ and we use the usual Euclidean ground metric $d(x,y)=|x-y|$.

In the 1D scalar-return setting, this distance has the quantile representation \cite{santambrogio2015optimal}
\begin{equation}
    W_p(\mu,\nu)
    =
    \left(
        \int_0^1
        \left|
            F_\mu^{-1}(\tau)-F_\nu^{-1}(\tau)
        \right|^p
        d\tau
    \right)^{1/p}.
\end{equation}

Thus, in one dimension, the Wasserstein distance is exactly an $L^p$ error between quantile functions. This perspective underlies quantile-based methods in \ac{DRL}.

\subsection{Monotone optimal transport and empirical sorting}
\label{appendix:monotone_ot_background}

In one dimension, an optimal transport coupling between two distributions is given by monotone rearrangement. Equivalently, the $\tau$-quantile of the source distribution is paired with the $\tau$-quantile of the target distribution. For empirical samples, this coupling is obtained by sorting both sets of samples and pairing them by rank.

Given source samples $\{x_{0,k}\}_{k=1}^K$ and target samples $\{x_{1,k}\}_{k=1}^K$, let
\[
    x_{0,(1)}\le \cdots \le x_{0,(K)},
    \qquad
    x_{1,(1)}\le \cdots \le x_{1,(K)}
\]
denote their order statistics. The empirical monotone coupling pairs $x_{0,(k)}$ with $x_{1,(k)}$. This sorted pairing is a finite-sample approximation to the population quantile coupling. In FlowIQN, the source samples are generated from quantile fractions and the target samples are Bellman return samples, so this sorting step converts sampled Bellman targets into an empirical quantile-supervised flow-matching objective.

\subsection{Conditional flow matching}
\label{appendix:cfm_background}

Flow matching learns a time-dependent velocity field whose induced ordinary differential equation transports a simple source distribution to a target distribution. Given a source sample $x_0\sim p_0$, a target sample $x_1\sim p_1$, and $t\sim\mathrm{Unif}([0,1])$, the standard linear interpolant is
\begin{equation}
    x_t = (1-t)x_0 + t x_1.
\end{equation}
The corresponding conditional target velocity is constant along the path:
\begin{equation}
    u_t(x_t\mid x_0,x_1)=x_1-x_0.
\end{equation}
The usual conditional flow-matching objective regresses a learned velocity field onto this target:
\begin{equation}
    \mathcal L_{\mathrm{CFM}}(\theta)
    =
    \mathbb E_{x_0,x_1,t}
    \left[
        \left\|
            v_\theta(t,x_t) - (x_1-x_0)
        \right\|_2^2
    \right].
\end{equation}

The coupling between $x_0$ and $x_1$ matters. With independent source-target sampling, the objective matches the target distribution but does not force a particular transport map. FlowIQN changes only this coupling: it keeps the same linear flow-matching regression, but pairs source and target samples using the empirical monotone coupling in Appendix~\ref{appendix:monotone_ot_background}. This is the step that makes the critic update compatible with the 1D Wasserstein geometry used in \ac{DRL}.

\subsection{Shortcut models}
\label{appendix:shortcut_background}

A standard flow model requires numerical integration of the learned velocity field from $t=0$ to $t=1$. Shortcut models reduce this cost by conditioning the network on a step size $d$ and training it to predict the average velocity over a finite interval. If $s_\theta(t,z,d,\cdot)$ denotes the shortcut field, the update is
\begin{equation}
    z_{t+d}
    =
    z_t + d\,s_\theta(t,z_t,d,\cdot).
\end{equation}
The model is trained with a self-consistency condition: one large step of size $2d$ should agree with two sequential steps of size $d$. In the notation used in this paper, this takes the form
\begin{equation}
    s_\theta(t,z_t,2d,\cdot)
    \approx
    \frac{1}{2}
    \left[
        s_{\bar\theta}(t,z_t,d,\cdot)
        +
        s_{\bar\theta}(t+d,z_{t+d},d,\cdot)
    \right],
\end{equation}
where $\bar\theta$ denotes a target network and $z_{t+d}=z_t+d\,s_{\bar\theta}(t,z_t,d,\cdot)$.

For FlowIQN, the relevant point is that the monotone-coupled target paths are straight lines. Therefore, the target velocity is constant along the path, and the ideal quantile transport already satisfies this shortcut consistency relation. The shortcut objective is therefore introduced as an efficiency mechanism rather than as a different Bellman projection target.

\section{Theoretical Results and Proofs}
\label{appendix:theory_and_proofs}

This appendix provides full proofs for the main-text theoretical claims. Throughout, we work in the 1D return setting, where Wasserstein distances admit a quantile representation. To facilitate analysis, we first consider an idealised population setting in which the monotone quantile coupling and target quantile function $F_\nu^{-1}(\cdot\mid s,a)$ are known exactly. Accordingly, losses are expressed as expectations over $\tau$ and $t$, rather than as the empirical mini-batch approximations used in Section~\ref{sec:method}. Proposition~\ref{prop:w2_upper_bound} therefore bounds the idealised population objective; the sorted mini-batch loss used by FlowIQN is a finite-sample approximation to this bound, subject to sampling and finite-batch effects.

\subsection{Setup and notation}

For a state-action pair $(s,a)$, let $\nu(\cdot\mid s,a)$ denote the Bellman target distribution and let $\eta_\theta(\cdot\mid s,a)$ denote the distribution induced by the FlowIQN critic. Let $F_\nu^{-1}(\cdot\mid s,a)$ denote the target quantile function. Since the source distribution is $\mathrm{Unif}[0,1]$, we identify each source sample with a quantile fraction $\tau\in[0,1]$.

We write the induced transport map as
\[
T_\theta(\tau\mid s,a)
=
g(\tau)+\int_0^1 v_\theta(t,z_t^\theta|s,a,\tau)\,dt,
\]
where $z_t^\theta$ denotes the corresponding flow trajectory, solving $\dot z_t^\theta=v_\theta(t,z_t^\theta\mid s,a,\tau)$ from $z_0^\theta=g(\tau)$, and $\eta_\theta=(T_\theta)_\#\mathrm{Unif}[0,1]$.

For ease of reading, we restate the 1D formula for the $p$-Wasserstein distance:
\begin{equation}
    W_p(\mu,\nu)
    =
    \left(
        \int_0^1
        \left|
            F_\mu^{-1}(\tau) - F_\nu^{-1}(\tau)
        \right|^p
        \, d\tau
    \right)^{1/p}.
    \label{eq:wasserstein_1d_appendix}
\end{equation}

Under the monotone quantile-coupled construction, we interpret $T_\theta(\cdot\mid s,a)$ as the model quantile map, so that Wasserstein error can be written directly as an $L^2$ error between the learned transport and the target quantile function. This is exact when $T_\theta(\cdot\mid s,a)$ is non-decreasing in $\tau$, and an upper bound otherwise. More generally, for any maps $T,S:[0,1]\to\mathbb R$, the pair $(T(\tau),S(\tau))$ with $\tau\sim\mathrm{Unif}[0,1]$ is a coupling of $T_\#\mathrm{Unif}[0,1]$ and $S_\#\mathrm{Unif}[0,1]$, and $W_2^2$ is the infimum of $\mathbb E|X-Y|^2$ over all couplings, so
\begin{equation}
W_2^2\bigl(T_\#\mathrm{Unif}[0,1],\,S_\#\mathrm{Unif}[0,1]\bigr)\le\int_0^1|T(\tau)-S(\tau)|^2\,d\tau,
\label{eq:coupling_bound}
\end{equation}
with equality when $T$ and $S$ are both non-decreasing, in which case they are the quantile functions in Eq.~\eqref{eq:wasserstein_1d_appendix}. We use Eq.~\eqref{eq:coupling_bound} for every Wasserstein bound below.

Likewise, under monotone coupling, the target path is the straight-line interpolation between $g(\tau)$ and $F_\nu^{-1}(\tau\mid s,a)$,
\[
z_t^\star=(1-t)\,g(\tau)+t\,F_\nu^{-1}(\tau\mid s,a),
\]
with constant target velocity
\[
u^\star(\tau\mid s,a)=F_\nu^{-1}(\tau\mid s,a)-g(\tau).
\]
The population FlowIQN loss at $(s,a)$ regresses the velocity field onto this target velocity \emph{along the target path},
\begin{equation}
\mathcal L_{\mathrm{FlowIQN}}(\theta;s,a)
=
\mathbb E_{\tau\sim\mathrm{Unif}[0,1],\,t\sim\mathrm{Unif}[0,1]}
\Bigl[\bigl|v_\theta(t,z_t^\star\mid s,a,\tau)-u^\star(\tau\mid s,a)\bigr|^2\Bigr],
\label{eq:flowiqn_population_loss}
\end{equation}
which is the population counterpart of Eq.~\eqref{eq:flowiqn_loss}. The model and target paths share the starting point $g(\tau)$ but in general differ for $t>0$, coinciding for all $t$ only when the loss is zero.

\subsection{Proof of Proposition~\ref{prop:w2_upper_bound}}
\label{app:proof_prop1}

\begin{proof}
Fix $(s,a)$ and $\tau\in[0,1]$. To simplify notation, we write $v_\theta(t,z)$ for $v_\theta(t,z\mid s,a,\tau)$. Define the path gap and the velocity error along the target path,
\[
\delta(t):=\bigl|z_t^\theta-z_t^\star\bigr|,
\qquad
e(t):=\bigl|v_\theta(t,z_t^\star)-u^\star(\tau\mid s,a)\bigr|.
\]
Both paths start at $g(\tau)$, so $\delta(0)=0$, and they end at $z_1^\theta=T_\theta(\tau\mid s,a)$ and $z_1^\star=F_\nu^{-1}(\tau\mid s,a)$, so $\delta(1)=\bigl|T_\theta(\tau\mid s,a)-F_\nu^{-1}(\tau\mid s,a)\bigr|$.

\emph{Step 1 (differential inequality).} Since $\dot z_t^\theta=v_\theta(t,z_t^\theta)$ and $\dot z_t^\star=u^\star(\tau\mid s,a)$, adding and subtracting $v_\theta(t,z_t^\star)$ and using that $v_\theta$ is $L$-Lipschitz in $z$ gives
\[
\bigl|\dot z_t^\theta-\dot z_t^\star\bigr|
\le
\bigl|v_\theta(t,z_t^\theta)-v_\theta(t,z_t^\star)\bigr|
+
\bigl|v_\theta(t,z_t^\star)-u^\star(\tau\mid s,a)\bigr|
\le
L\,\delta(t)+e(t).
\]
As the absolute value of an absolutely continuous function, $\delta$ is absolutely continuous with $\delta'(t)\le\bigl|\dot z_t^\theta-\dot z_t^\star\bigr|$ for almost every $t$, hence $\delta'(t)\le L\,\delta(t)+e(t)$ almost everywhere.

\emph{Step 2 (Gr\"onwall).} Multiplying by $e^{-Lt}$, integrating over $[0,1]$ and using $\delta(0)=0$ yields
\[
\delta(1)\le\int_0^1 e^{L(1-t)}\,e(t)\,dt.
\]

\emph{Step 3 (Cauchy--Schwarz).}
\[
\delta(1)^2
\le
\int_0^1 e^{2L(1-t)}\,dt\,\int_0^1 e(t)^2\,dt
=
C_L\int_0^1 e(t)^2\,dt,
\qquad
C_L:=\int_0^1 e^{2L(1-t)}\,dt=\frac{e^{2L}-1}{2L}.
\]

\emph{Step 4 (integrating over $\tau$).} The maps $T_\theta(\cdot\mid s,a)$ and $F_\nu^{-1}(\cdot\mid s,a)$ push $\mathrm{Unif}[0,1]$ forward to $\eta_\theta(\cdot\mid s,a)$ and $\nu(\cdot\mid s,a)$, respectively. Applying Eq.~\eqref{eq:coupling_bound} and then Step 3,
\begin{align*}
W_2^2\!\bigl(\eta_\theta(\cdot\mid s,a),\nu(\cdot\mid s,a)\bigr)
&\le
\int_0^1\bigl|T_\theta(\tau\mid s,a)-F_\nu^{-1}(\tau\mid s,a)\bigr|^2\,d\tau\\
&\le
C_L\,\mathbb E_{\tau}\Bigl[\int_0^1 e(t)^2\,dt\Bigr]
=
C_L\,\mathcal L_{\mathrm{FlowIQN}}(\theta;s,a),
\end{align*}
where the final equality is Eq.~\eqref{eq:flowiqn_population_loss}, using $t\sim\mathrm{Unif}[0,1]$. Finally, $C_L\ge1$ for all $L>0$, with $C_L\to1$ as $L\to0$.
\end{proof}

\paragraph{Condition (ii) of Definition~\ref{def:wasserstein_aligned}.}
\label{app:wa_discussion}
The proof above uses that the regression target $u^\star(\tau\mid s,a)$ is a single deterministic value for each $\tau$. The velocity field $v_\theta(t,z\mid s,a,\tau)=u^\star(\tau\mid s,a)$ therefore attains zero loss and transports $g(\tau)$ to $F_\nu^{-1}(\tau\mid s,a)$, so $\eta_\theta=\nu$ and condition (ii) holds. Under the independently coupled objective of Eq.~\eqref{eq:ret_cfm_loss}, the regression target $y-z_0$ remains random given $(t,z_t)$ whenever $\nu(\cdot\mid s,a)$ is not a point mass, so the loss contains a strictly positive conditional variance even when $\eta_\theta=\nu$, and condition (ii) fails. Condition (ii) is what separates the two methods: under the same Lipschitz assumption, Benton et al.~\cite{benton2024error} bound $W_2$ for independently coupled flow matching by the loss in excess of this variance, so (i) alone holds for both. This excess loss cannot be evaluated during training because the variance term is unknown. The monotone coupling removes this term, so that the FlowIQN loss itself controls the Wasserstein projection error.

\subsection{Proof of Corollary~\ref{cor:approx_proj}}
\label{app:proof_cor1}

\begin{proof}
We decompose the error to the fixed point into an approximate-projection term and a Bellman-contraction term. By the triangle inequality,
\[
\bar d_p(Z_{k+1},Z^\pi)
\le
\underbrace{\bar d_p(Z_{k+1},\mathcal T^\pi Z_k)}_{\text{approximate-projection term}}
+
\underbrace{\bar d_p(\mathcal T^\pi Z_k,\mathcal T^\pi Z^\pi)}_{\text{Bellman-contraction term}}.
\]
The corollary assumes that the critic update from $Z_k$ to $Z_{k+1}$ acts as an approximate projection of the Bellman target $\mathcal T^\pi Z_k$ back into the critic class, with residual error at most $\varepsilon_k$. That is,
\[
\bar d_p(Z_{k+1},\mathcal T^\pi Z_k)\le \varepsilon_k.
\]
Thus, the first term is controlled directly by the quality of the projection step.
For the second term, using the fixed-point relation \(Z^\pi=\mathcal T^\pi Z^\pi\) and the \(\gamma\)-contraction of \(\mathcal T^\pi\) in \(\bar d_p\), we obtain
\[
\bar d_p(\mathcal T^\pi Z_k,\mathcal T^\pi Z^\pi)
\le
\gamma\,\bar d_p(Z_k,Z^\pi).
\]
Combining the two bounds yields
\[
\bar d_p(Z_{k+1},Z^\pi)
\le
\varepsilon_k+\gamma\,\bar d_p(Z_k,Z^\pi).
\]
\end{proof}

\subsection{Optional: teacher--student distillation as Wasserstein minimization}
\label{appendix:distill_w2}

If a separate one-step student critic $Q_\phi(s,a,\tau)$ is used for efficient actor updates, its
distillation objective is also Wasserstein-aligned in the 1D setting, with respect to the teacher distribution and with $C_{\mathcal M}=1$.

\begin{proposition}
\label{prop:distill_w2}
For a state-action pair $(s,a)$, let $\eta_\theta(\cdot\mid s,a)$ and $\eta_\phi(\cdot\mid s,a)$ denote the
distributions induced by the teacher and student critics, respectively, where both are parameterized
by the shared quantile input $\tau\sim \mathrm{Unif}[0,1]$, and let
\[
\mathcal L_{\mathrm{distill}}(\phi; s,a)
:=
\mathbb E_{\tau\sim\mathrm{Unif}[0,1]}
\Big[
\big|
Q_\phi(s,a,\tau)
-
Q^{\mathrm{flow}}_\theta(s,a,g(\tau),\tau)
\big|^2
\Big].
\]
Then $\mathcal L_{\mathrm{distill}}(\phi; s,a)$ is exactly $W_2^2\bigl(\eta_\phi(\cdot\mid s,a),\eta_\theta(\cdot\mid s,a)\bigr)$ when $Q_\phi(s,a,\cdot)$ and $Q^{\mathrm{flow}}_\theta(s,a,g(\cdot),\cdot)$ are both non-decreasing in $\tau$, and an upper bound on it otherwise. Consequently, averaging this loss over state-action pairs minimizes the expected squared $2$-Wasserstein discrepancy between the student and teacher critics, or an upper bound on it when the maps are not monotone.
\end{proposition}

\begin{proof}
By construction, $\eta_\phi(\cdot\mid s,a)$ and $\eta_\theta(\cdot\mid s,a)$ are the pushforwards of $\mathrm{Unif}[0,1]$ under the maps $\tau\mapsto Q_\phi(s,a,\tau)$ and $\tau\mapsto Q^{\mathrm{flow}}_\theta(s,a,g(\tau),\tau)$, respectively. Applying Eq.~\eqref{eq:coupling_bound} to these two maps gives
\[
W_2^2\!\bigl(\eta_\phi(\cdot\mid s,a),\eta_\theta(\cdot\mid s,a)\bigr)
\le
\int_0^1
\left|
Q_\phi(s,a,\tau)-Q^{\mathrm{flow}}_\theta(s,a,g(\tau),\tau)
\right|^2
\,d\tau
=
\mathcal L_{\mathrm{distill}}(\phi; s,a),
\]
with equality when both maps are non-decreasing in $\tau$, in which case they are the quantile functions of $\eta_\phi(\cdot\mid s,a)$ and $\eta_\theta(\cdot\mid s,a)$. Hence $C_{\mathcal M}=1$. No Lipschitz assumption is needed, since neither map is compared along an integrated trajectory.
\end{proof}

\subsection{Proof of Proposition~\ref{prop:shortcut_bias}}
\label{app:proof_prop2}

\begin{proof}
Under monotone coupling, the target path between source and target quantiles is
\[
z_t^\star=(1-t)g(\tau)+tF_\nu^{-1}(\tau\mid s,a),
\]
which is linear in $t$. Therefore its velocity is constant:
\[
\frac{d}{dt}z_t^\star = F_\nu^{-1}(\tau\mid s,a)-g(\tau)=u^\star(\tau\mid s,a).
\]

A shortcut model represents the average velocity over a step of size $d$. For the linear target path above, this average velocity is the same for every start time $t$ and every step size $d$, namely $u^\star(\tau\mid s,a)$. Hence the one-step prediction over size $2d$ is
\[
z_{t+2d}^\star = z_t^\star + 2d\,u^\star(\tau\mid s,a),
\]
while composing two shortcut steps of size $d$ yields
\[
z_{t+d}^\star = z_t^\star + d\,u^\star(\tau\mid s,a),
\qquad
z_{t+2d}^\star = z_{t+d}^\star + d\,u^\star(\tau\mid s,a)
= z_t^\star + 2d\,u^\star(\tau\mid s,a).
\]
Thus one large step and two small steps coincide exactly. Equivalently, the target transport satisfies the shortcut self-consistency relation with zero residual. Therefore adding the consistency loss does not change the optimal target transport solution.
\end{proof}

\subsection{Proof of Corollary~\ref{cor:shortcut_w2}}
\label{app:proof_cor2}

\begin{proof}
For $d=1$, the shortcut model maps the source quantile input $g(\tau)$ directly to a predicted return sample in a single step. Let $\hat T_\theta(\tau\mid s,a)$ denote the resulting one-step shortcut output map. Under monotone coupling, the target output map is the target quantile function $F_\nu^{-1}(\tau\mid s,a)$.

The one-step shortcut training objective therefore takes the form
\[
\mathcal L_{\mathrm{shortcut}}^{d=1}(\theta; s,a)
=
\mathbb E_{\tau\sim\mathrm{Unif}[0,1]}
\left[
\left|
\hat T_\theta(\tau\mid s,a)-F_\nu^{-1}(\tau\mid s,a)
\right|^2
\right].
\]
The maps $\hat T_\theta(\cdot\mid s,a)$ and $F_\nu^{-1}(\cdot\mid s,a)$ push $\mathrm{Unif}[0,1]$ forward to $\eta_\theta^{\mathrm{shortcut}}(\cdot\mid s,a)$ and $\nu(\cdot\mid s,a)$, respectively, so Eq.~\eqref{eq:coupling_bound} gives
\[
W_2^2\!\bigl(\eta_\theta^{\mathrm{shortcut}}(\cdot\mid s,a),\nu(\cdot\mid s,a)\bigr)
\le
\mathcal L_{\mathrm{shortcut}}^{d=1}(\theta; s,a),
\]
with equality when $\hat T_\theta(\cdot\mid s,a)$ is non-decreasing in $\tau$, since $F_\nu^{-1}(\cdot\mid s,a)$ is non-decreasing by definition. Thus, in the single-step case, the shortcut objective is exactly the squared $2$-Wasserstein distance to the Bellman target when the one-step map is monotone in $\tau$, and an upper bound on it otherwise, so the single-step shortcut critic is Wasserstein-aligned with $C_{\mathcal M}=1$. No Lipschitz assumption is needed: the one-step map evaluates the velocity field only at $t=0$, where the model and target paths coincide at $g(\tau)$, so the path mismatch handled in the proof of Proposition~\ref{prop:w2_upper_bound} does not arise.
\end{proof}

\subsection{Adaptive time discretization and discretization error}
\label{appendix:time_schedule}

The FlowIQN critic defines a continuous-time transport, but critic evaluation requires numerically integrating this transport with a finite number of solver steps. In practice, we observe that the learned flow curvature is not constant across time: curvature is often concentrated near the beginning of the transport, while later portions of the trajectory tend to be smoother. Since Euler discretization error is largest in these high-curvature regions, a uniform time grid can allocate computation inefficiently.
To address this, we use an adaptive integration schedule inspired by recent adaptive-solver viewpoints for generative ODEs \cite{jo2026formalizing}.

Specifically, we periodically estimate a curvature proxy
$c(t) \approx \mathbb{E}_{(s,a),\tau}[\,| \tfrac{d}{dt}v_\theta(t,Z_t|s,a,\tau) |\,]$
on minibatches using Jacobian-vector products, smooth it with an exponential
moving average, and allocate timesteps according to the cumulative profile
\begin{equation}
    C(t)
    =
    \frac{\int_0^t \sqrt{c(\xi)+\epsilon}\, d\xi}
         {\int_0^1 \sqrt{c(\xi)+\epsilon}\, d\xi},
    \qquad
    t_m = C^{-1}\!\left(\frac{m}{M}\right),
\end{equation}
with $0=t_0 < \cdots < t_M=1$. This places more solver steps in high-curvature
regions and fewer where the transport is nearly linear. The training objective
is unchanged; the adaptive schedule only affects numerical integration during
target construction and inference. 

\section{FlowIQN in Offline RL}
\label{appendix:flowiqn_offline_rl}

FlowIQN is a critic-learning method rather than a policy extraction method. In the offline setting, this means that the main additional components beyond the critic objective itself are: \emph{(i)} Bellman target construction from a fixed dataset, and \emph{(ii)} a policy extraction rule that converts the learned return distribution into actions. In our experiments, we instantiate FlowIQN with two such policy extraction mechanisms: an FQL-style learned actor \cite{fql_park2025}, and rejection sampling from a behavior-cloned proposal policy. This appendix describes the shared critic update used in both cases, and then details each instantiation separately.

\subsection{Bellman target construction for offline FlowIQN}
\label{appendix:offline_bellman_targets}

We assume a static offline dataset
\[
\mathcal D = \{(s,a,r,s',m)\},
\]
where \(m \in \{0,1\}\) indicates whether the transition is non-terminal. At each iteration, we sample a minibatch
\[
B=\{(s_i,a_i,r_i,s'_i,m_i)\}_{i=1}^N
\]
from \(\mathcal D\). As in standard offline RL, critic bootstrapping uses a target critic \(Q_{\bar\theta}^{\mathrm{flow}}\), implemented as an exponential moving average of the online FlowIQN parameters.

For each transition \(i\), we first sample \(K\) quantile fractions
\[
\tau'_{i,k} \sim \mathrm{Unif}[0,1], \qquad k=1,\dots,K,
\]
map them to source values \(z'_{0,i,k} = g(\tau'_{i,k})\). In defining this mapping, we follow previous work \cite{agrawalla2026floq} and approximate
\begin{equation}
    Q_{\max} := \frac{r_{\max}}{1-\gamma},
    \qquad
    Q_{\min} := \frac{r_{\min}}{1-\gamma},
\end{equation}
where $r_{\min}$ and $r_{\max}$ denote the minimum and maximum rewards observed in the offline dataset $\mathcal D$. We then define the source interval by
\begin{equation}
    u := Q_{\max},
    \qquad
    l := Q_{\max} - \kappa (Q_{\max} - Q_{\min}),
\end{equation}
where $\kappa$ define a simple linear mapping
\begin{equation}
    g(\tau'_{i,k})
    :=
    l + \tau'_{i,k}(u-l).
\end{equation}
We evaluate the target critic at the next state and a next action \(a'_i\) chosen by the policy extraction rule under consideration:
\begin{equation}
    z'_{i,k}
    =
    Q_{\bar\theta}^{\mathrm{flow}}(s'_i,a'_i,z'_{0,i,k},\tau'_{i,k}).
\end{equation}
The resulting Bellman target samples are
\begin{equation}
    y_{i,k}
    =
    r_i + \gamma m_i z'_{i,k}.
\end{equation}
Thus, the offline target distribution is represented through quantile-indexed return samples produced by the target FlowIQN critic.
Given the target samples \(\{y_{i,k}\}_{k=1}^K\), we reuse the quantile fractions as the source, \(\tau_{i,k}=\tau'_{i,k}\), form \(z_{0,i,k}=g(\tau_{i,k})\), and sort the source and target samples within each transition:
\[
\tilde{\tau}_{i,1}\le \cdots \le \tilde{\tau}_{i,K},
\qquad
\tilde{z}_{0,i,1}\le \cdots \le \tilde{z}_{0,i,K},
\qquad
\tilde{y}_{i,1}\le \cdots \le \tilde{y}_{i,K}.
\]
We then construct the straight-line interpolants
\begin{equation}
    z_{i,k,t}
    =
    (1-t)\tilde{z}_{0,i,k} + t\tilde{y}_{i,k},
    \qquad
    t\sim \mathrm{Unif}[0,1],
\end{equation}
and minimize the quantile-coupled critic loss
\begin{equation}
\label{eq:appendix_flowiqn_offline_loss}
    \mathcal{L}_{\mathrm{FlowIQN}}
    =
    \frac{1}{NK}
    \sum_{i=1}^N
    \sum_{k=1}^K
    \mathbb{E}_{t}
    \Big[
        \big|
            v_\theta(t,z_{i,k,t},s_i,a_i,\tilde{\tau}_{i,k})
            -
            (\tilde{y}_{i,k}-\tilde{z}_{0,i,k})
        \big|^2
    \Big].
\end{equation}
This is exactly the same quantile-coupled projection described in the main method section, now applied to Bellman targets generated from an offline replay distribution.

Whenever a scalar action score is required for policy extraction, we use the empirical mean of the learned return distribution evaluated on a fixed quantile grid. Specifically, let
\begin{equation}
    \tau_k := \frac{k - 0.5}{K},
    \qquad k=1,\dots,K,
\end{equation}
then
\begin{equation}
\label{eq:appendix_scalar_q}
    \hat Q(s,a)
    :=
    \frac{1}{K}
    \sum_{k=1}^K Q(s,a,\tau_k).
\end{equation}
In our implementation, this deterministic grid is used whenever actions are ranked, avoiding sampling bias in policy extraction.
In our experiments, this risk-neutral scalarization is used throughout.

\subsection{Instantiation with FQL-style actor extraction}
\label{appendix:flowiqn_fql_instantiation}

Our first policy extraction instantiation combines the FlowIQN critic with the actor-learning framework of Flow Q-Learning (FQL) \cite{fql_park2025}, named throughout as FlowIQN-FQL. In FQL a flow-matching behavior policy is trained via behavioral cloning, while a separate one-step actor is trained with RL and regularized toward the behavior policy. This separation avoids backpropagation through the iterative flow-generation process during actor learning.
Concretely, we train a behavior-cloned flow policy \(\pi_{\mathrm{BC}}^{\mathrm{flow}}\) by standard conditional flow matching on dataset actions. Let \(v_\psi(t,x_t,s)\) denote its velocity field, and let \(\pi_{\mathrm{BC}}^{\mathrm{flow}}(s,x_0)\) denote the induced action obtained by numerically integrating this field from an initial noise sample \(x_0\). We then train a separate one-step actor \(\mu_\omega(s,x_0)\) to maximize the scalarized FlowIQN value while remaining close to the behavior policy:
\begin{equation}
    \mathcal{L}_{\mathrm{actor}}
    =
    \mathbb{E}_{s,x_0}
    \Big[
        -\hat Q_\phi(s,\mu_\omega(s,x_0))
        +
        \lambda
        \big\|
            \mu_\omega(s,x_0)
            -
            \mathrm{sg}\!\left(\pi_{\mathrm{BC}}^{\mathrm{flow}}(s,x_0)\right)
        \big\|^2
    \Big],
\end{equation}
where \(\hat Q_\phi\) is a scalar critic score computed from a distilled one-step student critic \(Q_\phi(s,a,\tau)\), and \(\mathrm{sg}(\cdot)\) denotes stop-gradient. 
To improve efficiency and avoid backpropagation through time, we follow previous work \cite{agrawalla2026floq} in training a single step student critic. 
We use this student critic only as an efficiency mechanism for repeated actor updates; the main critic learning principle remains the multi-step FlowIQN teacher updated by Eq.~\eqref{eq:appendix_flowiqn_offline_loss}.

The student critic is trained to match the flow teacher at shared quantile fractions:
\begin{equation}
    \mathcal{L}_{\mathrm{distill}}
    =
    \mathbb{E}_{(s,a),\,\tau}
    \Big[
        \big|
            Q_\phi(s,a,\tau)
            -
            \mathrm{sg}\!\left(
                Q_{\bar\theta}^{\mathrm{flow}}(s,a,g(\tau),\tau)
            \right)
        \big|^2
    \Big].
\end{equation}
In practice, this yields a cheap one-step critic for actor optimization while keeping the full flow teacher as the object trained by Bellman bootstrapping. For theoretical analysis of this approach, we refer the reader to Appendix \ref{appendix:theory_and_proofs}. A full description is provided in Algorithm~\ref{alg:flowiqn_fql_appendix}. 

\newcounter{savedLine}

\begin{algorithm}[htbp]
\caption{FlowIQN with FQL-style actor extraction \cite{fql_park2025}}
\label{alg:flowiqn_fql_appendix}
\begin{algorithmic}[1]
\State \textbf{Input:} Dataset $\mathcal D=\{(s,a,r,s',m)\}$
\State \textbf{Hyperparameters:} Quantiles $K$, flow steps $M$, source scale $\kappa$, actor regularization $\lambda$, $f_{\mathrm{sched}}$
\State \textbf{Networks:} FlowIQN critic velocity $v_\theta(t,z,s,a,\tau)$, target velocity $v_{\bar\theta}$, student critic $Q_\phi(s,a,\tau)$, BC flow velocity $v_\psi(t,x,s)$, one-step actor $\mu_\omega(s,x_0)$

\vspace{0.3em}
\State \textbf{Definitions:}

\State \hskip1.5em Let $Q_{\max}=r_{\max}/(1-\gamma)$ and $Q_{\min}=r_{\min}/(1-\gamma)$, where $r_{\max}$ and $r_{\min}$ are computed from $\mathcal D$.
\State \hskip1.5em Let $u=Q_{\max}$, $l=Q_{\max}-\kappa(Q_{\max}-Q_{\min})$, and define $g(\tau)=l+\tau(u-l)$.

\State \hskip1.5em For $z_0=g(\tau)$, define $Q^{\mathrm{flow}}_\theta(s,a,z_0,\tau)\coloneqq z_M$, where
\State \hskip3em $z_{j+1}=z_j+\frac{1}{M}v_\theta\!\left(\frac{j}{M},z_j,s,a,\tau\right)$ for $j=0,\dots,M-1$.

\State \hskip1.5em Define $Q^{\mathrm{flow}}_{\bar\theta}(s,a,z_0,\tau)\coloneqq z_M$, where
\State \hskip3em $z_{j+1}=z_j+\frac{1}{M}v_{\bar\theta}\!\left(\frac{j}{M},z_j,s,a,\tau\right)$ for $j=0,\dots,M-1$.

\State \hskip1.5em Define $\pi^{\mathrm{flow}}_{\mathrm{BC}}(s,x_0)\coloneqq x_M$, where
\State \hskip3em $x_{j+1}=x_j+\frac{1}{M}v_\psi\!\left(\frac{j}{M},x_j,s\right)$ for $j=0,\dots,M-1$.

\While{not converged}
    \State Sample batch $B=\{(s_i,a_i,r_i,s'_i,m_i)\}_{i=1}^N \sim \mathcal D$
    \setcounter{savedLine}{\value{ALG@line}}

    \saveALGstate

    \Statex \hskip -0.3em \colorbox{cbBlue!15}{%
      \begin{minipage}{\dimexpr\linewidth-2\fboxsep\relax}
      \vspace{0.3em}
      \textbf{\textcolor{cbBlue!85!black}{\hspace{0.5em} Shared FlowIQN critic update}}
      \vspace{0.3em}

      \begin{algorithmic}[1]
      \setcounter{ALG@line}{\value{savedLine}}

      \For{each transition $i \in \{1,\dots,N\}$}
          \State Sample $x_0 \sim \mathcal N(0,I)$ and set $a'_i \leftarrow \mu_\omega(s'_i,x_0)$
          \State Sample $\tau'_{i,k} \sim U([0,1])$ and set $z'_{0,i,k} \leftarrow g(\tau'_{i,k})$ for $k=1,\dots,K$
          \State $z'_{i,k} \leftarrow Q_{\bar\theta}^{\mathrm{flow}}(s'_i,a'_i,z'_{0,i,k},\tau'_{i,k})$
          \State $y_{i,k} \leftarrow r_i + \gamma m_i z'_{i,k}$
      \EndFor

      \For{each transition $i \in \{1,\dots,N\}$}
        \State Sort $\{\tilde y_{i,k}\}_{k=1}^K \leftarrow \mathrm{sort}(\{y_{i,k}\}_{k=1}^K)$
        \State Sort $\{\tilde\tau_{i,k}\}_{k=1}^K \leftarrow \mathrm{sort}(\{\tau'_{i,k}\}_{k=1}^K)$
        \State Set $\tilde z_{0,i,k} \leftarrow g(\tilde\tau_{i,k})$
    \EndFor

      \State Sample $t\sim U([0,1])$ and set $z_{i,k,t} \leftarrow (1-t)\tilde z_{0,i,k}+t\tilde y_{i,k}$
      \State $\mathcal L_{\mathrm{FlowIQN}}
= \mathbb E_{i,k,t}\!\left[
\left\|v_\theta(t,z_{i,k,t},s_i,a_i,\tilde\tau_{i,k})
-(\tilde y_{i,k}-\tilde z_{0,i,k})\right\|^2
\right]$ \Comment{Update $\theta$}
      \State $\mathcal L_{\mathrm{distill}}
= \mathbb E_{i,k}\!\left[
\left\|
Q_\phi(s_i,a_i,\tau'_{i,k})
-
\mathrm{sg}\!\left(
Q_{\bar\theta}^{\mathrm{flow}}(s_i,a_i,g(\tau'_{i,k}),\tau'_{i,k})
\right)
\right\|^2
\right]$ \Comment{Update $\phi$}

      \end{algorithmic}
      \vspace{0.1em}
      \end{minipage}%
    }

    \restoreALGstate

    \vspace{0.4em}
    \State \hspace{-\algorithmicindent}\textbf{FQL-style policy extraction}
    \State \hspace{-\algorithmicindent}Sample action $a\sim B$, noise $x_0\sim \mathcal N(0,I)$, and time $t\sim U([0,1])$
    \State \hspace{-\algorithmicindent}$x_1 \leftarrow a$; \quad $x_t \leftarrow (1-t)x_0+tx_1$
    \State \hspace{-\algorithmicindent}$\mathcal L_{\mathrm{BC}}=\mathbb E_{t,x_0,a}\!\left[\|v_\psi(x_t,t,s)-(x_1-x_0)\|^2\right]$ \Comment{Update $\psi$}
    \State \hspace{-\algorithmicindent}$a_\omega \leftarrow \mu_\omega(s,x_0)$; \quad $a_{\mathrm{BC}} \leftarrow \mathrm{sg}(\pi^{\mathrm{flow}}_{\mathrm{BC}}(s,x_0))$
    \State \hspace{-\algorithmicindent}$\mathcal L_{\mathrm{actor}}=\mathbb E_{s,x_0}\!\left[-\hat Q_\phi(s,a_\omega)+\lambda\|a_\omega-a_{\mathrm{BC}}\|^2\right]$ \Comment{Update $\omega$}
    \State \hspace{-\algorithmicindent}Update target critic parameters $\bar\theta$ with EMA
\State \hspace{-\algorithmicindent}\textbf{if} $k \bmod f_{\mathrm{sched}} = 0$
\State \hspace{1em}update amortized time discretization schedule (Appendix~\ref{appendix:time_schedule})
\EndWhile

\vspace{0.3em}
\Statex \hskip -1.5em \textbf{Return} one-step actor $\mu_\omega$
\end{algorithmic}
\end{algorithm}

\subsection{Instantiation with rejection sampling}
\label{appendix:flowiqn_rs_instantiation}

Our second policy extraction instantiation uses rejection sampling rather than learning a separate actor to maximise return, reffered to throughout as FlowIQN-R. This emphasizes that FlowIQN is not tied to any particular actor objective. Following prior offline RL work with expressive behavior models \cite{agrawalla2026floq, dongzheng2026value, fql_park2025}, we first train a behaviour policy \(\pi_{\mathrm{BC}}\) with behavioural cloning on the offline dataset. When sampling actions, we sample a set of candidate actions from this behaviour policy and rank them using the scalarized FlowIQN critic:
\begin{equation}
    a^{(1)},\dots,a^{(J)} \sim \pi_{\mathrm{BC}}(\cdot\mid s),
    \qquad
    a^\star
    =
    \arg\max_{j\in\{1,\dots,J\}}
    \hat Q_\theta(s,a^{(j)}),
\end{equation}
where \(Q_\theta\) denotes the empirical mean return predicted by the FlowIQN critic, as in Eq.~\eqref{eq:appendix_scalar_q}. In this variant, no separate actor is optimized against the critic. 

For bootstrapping, the next action \(a'_i\) in Sec.~\ref{appendix:offline_bellman_targets} is obtained in the same way: a set of candidates is sampled from the proposal model at \(s'_i\), and the candidate with the largest scalarized target-critic score is used to form the Bellman target distribution. A full description is provided in Algorithm~\ref{alg:flowiqn_rs_appendix}. 

\newcounter{savedLineRS}

\begin{algorithm}[htbp]
\caption{FlowIQN with rejection-sampling actor extraction}
\label{alg:flowiqn_rs_appendix}
\begin{algorithmic}[1]
\State \textbf{Input:} Dataset $\mathcal D=\{(s,a,r,s',m)\}$
\State \textbf{Hyperparameters:} Quantiles $K$, flow steps $M$, source scale $\kappa$, candidates $J$, $f_{\mathrm{sched}}$
\State \textbf{Networks:} FlowIQN critic velocity $v_\theta(t,z,s,a,\tau)$, target velocity $v_{\bar\theta}$, BC flow velocity $v_\psi(t,x,s)$

\vspace{0.3em}
\State \textbf{Definitions:}

\State \hskip1.5em Let $Q_{\max}=r_{\max}/(1-\gamma)$ and $Q_{\min}=r_{\min}/(1-\gamma)$, where $r_{\max}$ and $r_{\min}$ are computed from $\mathcal D$.
\State \hskip1.5em Let $u=Q_{\max}$, $l=Q_{\max}-\kappa(Q_{\max}-Q_{\min})$, and define $g(\tau)=l+\tau(u-l)$.

\State \hskip1.5em For $z_0=g(\tau)$, define $Q^{\mathrm{flow}}_\theta(s,a,z_0,\tau)\coloneqq z_M$, where
\State \hskip3em $z_{j+1}=z_j+\frac{1}{M}v_\theta\!\left(\frac{j}{M},z_j,s,a,\tau\right)$ for $j=0,\dots,M-1$.

\State \hskip1.5em Define $Q^{\mathrm{flow}}_{\bar\theta}(s,a,z_0,\tau)\coloneqq z_M$, where
\State \hskip3em $z_{j+1}=z_j+\frac{1}{M}v_{\bar\theta}\!\left(\frac{j}{M},z_j,s,a,\tau\right)$ for $j=0,\dots,M-1$.

\State \hskip1.5em Define $\pi^{\mathrm{flow}}_{\mathrm{BC}}(s,x_0)\coloneqq x_M$, where
\State \hskip3em $x_{j+1}=x_j+\frac{1}{M}v_\psi\!\left(\frac{j}{M},x_j,s\right)$ for $j=0,\dots,M-1$.

\State \hskip1.5em Define the empirical flow value estimate
\[
\hat Q^{\mathrm{flow}}_\theta(s,a)
\coloneqq
\frac{1}{K}\sum_{k=1}^K
Q^{\mathrm{flow}}_\theta(s,a,g(\tau_k),\tau_k),
\qquad
\tau_k\sim U([0,1]).
\]

\While{not converged}
    \State Sample batch $B=\{(s_i,a_i,r_i,s'_i,m_i)\}_{i=1}^N \sim \mathcal D$
    \setcounter{savedLineRS}{\value{ALG@line}}

    \saveALGstate

    \Statex \hskip -0.3em \colorbox{cbBlue!15}{%
      \begin{minipage}{\dimexpr\linewidth-2\fboxsep\relax}
      \vspace{0.3em}
      \textbf{\textcolor{cbBlue!85!black}{\hspace{0.5em} Shared FlowIQN critic update}}
      \vspace{0.3em}

      \begin{algorithmic}[1]
      \setcounter{ALG@line}{\value{savedLineRS}}

      \For{each transition $i \in \{1,\dots,N\}$}
          \State Sample $x_{0,i}^{(j)}\sim\mathcal N(0,I)$ and set $a_i^{\prime(j)}\leftarrow \pi^{\mathrm{flow}}_{\mathrm{BC}}(s'_i,x_{0,i}^{(j)})$ for $j=1,\dots,J$
          \State Select $j_i^\star \leftarrow \arg\max_{j\in\{1,\dots,J\}}\hat Q^{\mathrm{flow}}_{\bar\theta}(s'_i,a_i^{\prime(j)})$ and set $a'_i\leftarrow a_i^{\prime(j_i^\star)}$
          \State Sample $\tau'_{i,k} \sim U([0,1])$ and set $z'_{0,i,k} \leftarrow g(\tau'_{i,k})$ for $k=1,\dots,K$
          \State $z'_{i,k} \leftarrow Q_{\bar\theta}^{\mathrm{flow}}(s'_i,a'_i,z'_{0,i,k},\tau'_{i,k})$
          \State $y_{i,k} \leftarrow r_i + \gamma m_i z'_{i,k}$
      \EndFor

      \For{each transition $i \in \{1,\dots,N\}$}
        \State Sort $\{\tilde y_{i,k}\}_{k=1}^K \leftarrow \mathrm{sort}(\{y_{i,k}\}_{k=1}^K)$
        \State Sort $\{\tilde\tau_{i,k}\}_{k=1}^K \leftarrow \mathrm{sort}(\{\tau'_{i,k}\}_{k=1}^K)$
        \State Set $\tilde z_{0,i,k} \leftarrow g(\tilde\tau_{i,k})$
      \EndFor

      \State Sample $t\sim U([0,1])$ and set $z_{i,k,t} \leftarrow (1-t)\tilde z_{0,i,k}+t\tilde y_{i,k}$
      \State $\mathcal L_{\mathrm{FlowIQN}}
= \mathbb E_{i,k,t}\!\left[
\left\|v_\theta(t,z_{i,k,t},s_i,a_i,\tilde\tau_{i,k})
-(\tilde y_{i,k}-\tilde z_{0,i,k})\right\|^2
\right]$ \Comment{Update $\theta$}

      \end{algorithmic}
      \vspace{0.1em}
      \end{minipage}%
    }

    \restoreALGstate

    \vspace{0.4em}
    \State \hspace{-\algorithmicindent}\textbf{Rejection-sampling policy extraction}
    \State \hspace{-\algorithmicindent}Sample action $a\sim B$, noise $x_0\sim \mathcal N(0,I)$, and time $t\sim U([0,1])$
    \State \hspace{-\algorithmicindent}$x_1 \leftarrow a$; \quad $x_t \leftarrow (1-t)x_0+tx_1$
    \State \hspace{-\algorithmicindent}$\mathcal L_{\mathrm{BC}}=\mathbb E_{t,x_0,a}\!\left[\|v_\psi(t,x_t,s)-(x_1-x_0)\|^2\right]$ \Comment{Update $\psi$}
    \State \hspace{-\algorithmicindent}Update target critic parameters $\bar\theta$ with EMA
    \State \hspace{-\algorithmicindent}\textbf{if} $n \bmod f_{\mathrm{sched}} = 0$
    \State \hspace{1em}update amortized time discretization schedule (Appendix~\ref{appendix:time_schedule})
\EndWhile

\vspace{0.3em}
\Statex \hskip -1.5em \textbf{Return} proposal policy $\pi^{\mathrm{flow}}_{\mathrm{BC}}$ and action rule
\[
a^\star=a^{(j^\star)},
\qquad
j^\star
=
\arg\max_{j\in\{1,\dots,J\}}
\hat Q^{\mathrm{flow}}_\theta(s,a^{(j)}),
\]
\[
a^{(j)}=\pi^{\mathrm{flow}}_{\mathrm{BC}}(s,x_0^{(j)}),
\qquad
x_0^{(j)}\sim\mathcal N(0,I).
\]
\end{algorithmic}
\end{algorithm}

This instantiation avoids the need to learn an actor against the critic altogether, and therefore provides a useful complementary view of FlowIQN: the critic can be used either to train a policy, or simply to rank candidate actions from a support-constrained proposal mechanism.

\section{Experimental Setup}
\label{appendix:experimental_setup}

\subsection{Fixed-policy return distribution evaluation}
\label{appendix:w2_policy_eval}

Evaluating return distribution accuracy in actor-critic methods requires care because the return distribution is policy-dependent. If a critic is trained while its own policy is improving, then it is trained to approximate $Z^{\pi_m}(s,a)$ for the method's current policy $\pi_m$. Comparing that critic to MC returns collected under a different reference policy $\pi_{\mathrm{ref}}$ instead evaluates it against $Z^{\pi_{\mathrm{ref}}}(s,a)$. The resulting error therefore mixes critic approximation error with policy mismatch:
\begin{equation}
    W_2\!\left(\hat Z_m(s,a), Z^{\pi_{\mathrm{ref}}}(s,a)\right),
\end{equation}
which is not, in general, the same as the desired critic error
\begin{equation}
    W_2\!\left(\hat Z_m(s,a), Z^{\pi_m}(s,a)\right).
\end{equation}
For example, a critic paired with a poor policy may correctly predict low returns, but would appear inaccurate if evaluated against MC returns from a stronger reference policy. Thus, such evaluations can conflate return-distribution accuracy with downstream policy quality.

To isolate critic learning, we use a fixed-policy protocol. We first train an FQL agent for $10^6$ gradient steps and use the resulting policy $\pi_{\mathrm{FQL}}$ as the fixed evaluation policy. All critic baselines are then trained without policy updates, using Bellman targets induced by this same frozen policy, and are evaluated against MC return distributions collected under $\pi_{\mathrm{FQL}}$. This ensures that differences in empirical $W_2$ reflect differences in critic approximation rather than differences in the policies whose returns are being predicted.

We collect $20$ trajectories under $\pi_{\mathrm{FQL}}$, each typically around $750$ environment steps long. Rather than evaluating only at uniformly random state-action pairs, we use a stratified sampling procedure designed to include both typical states and challenging decision points. The OGBench scene tasks use a semi-sparse reward structure in which the reward changes when a subtask is completed or undone. We identify timesteps where the reward changes by $\pm 1$ and select the ten state-action pairs immediately leading up to and including each such reward change. To avoid evaluating only near reward discontinuities, we also sample state-action pairs uniformly along each trajectory, taking one pair every $50$ timesteps. This procedure yields $1{,}578$ state-action pairs.

For each selected pair $(s,a)$, we estimate the target return distribution by MC rollouts under the fixed policy. We restore the simulator to the corresponding state, execute the selected action $a$, and then continue with actions sampled from $\pi_{\mathrm{FQL}}$. We collect $200$ rollout returns per selected pair, using a maximum horizon of $688$ and discount factor $\gamma=0.99$:
\begin{equation}
    G_H(s,a) = \sum_{t=0}^{H-1} \gamma^t r_t,
    \qquad H=688.
\end{equation}
Since $\gamma^{688}\approx 10^{-3}$, rewards beyond the rollout horizon are strongly discounted. The resulting empirical return distribution is used as the MC target for evaluating critic accuracy.

For each method and seed, we evaluate the learned return distribution at the $1{,}578$ selected state-action pairs and compute the empirical $2$-Wasserstein distance to the corresponding MC target distribution. We treat each selected state-action pair as a separate evaluation task and aggregate results across tasks and seeds using performance profiles and the interquartile mean (IQM). Since \text{rliable} assumes larger values are better, we report performance as negative empirical $W_2$ distance.

\subsection{Offline RL Baselines}
\label{appendix:offline_rl_baselines}
We choose baselines to isolate the contribution of the proposed critic-learning principle. Because policy extraction can itself strongly affect offline RL performance, we control for this where possible by using a shared FQL-style extraction rule \cite{fql_park2025}. Specifically, following FQL, we use a behaviour-cloned flow policy and a one-step actor trained to maximize critic values under distillation-based behavioural regularisation. This extraction setup is used for \textsc{IQN}, \textsc{CODAC}, and our base FlowIQN-FQL algorithm (see Appendix~\ref{appendix:flowiqn_fql_instantiation}), so that differences between these methods are driven primarily by the critic rather than the actor. The same FQL-style policy extraction framework is also used by \textsc{floq}. We also evaluate \textsc{FlowIQN-R}, a rejection-sampling variant of FlowIQN detailed in Appendix~\ref{appendix:flowiqn_rs_instantiation}.

Under this common actor construction, \textsc{IQN} provides a canonical non-flow, quantile-based distributional critic baseline, approximating the return distribution by predicting quantile values at sampled quantile fractions. This makes it the closest classical \ac{DRL} comparison to FlowIQN's Wasserstein-aligned critic objective. \textsc{CODAC} provides a stronger distributional offline RL baseline with additional conservative regularisation, and therefore tests whether FlowIQN remains competitive against distributional offline RL methods that additionally incorporate explicit pessimism. Beyond these controlled critic comparisons, we include \textsc{BC} as the simplest imitation-only floor, as it learns by maximizing the likelihood of dataset actions at each state. We also include \textsc{ReBRAC} as a strong standard offline \ac{RL} baseline with a Gaussian policy and non-distributional scalar critic. We rereport \textsc{CODAC} from \cite{dongzheng2026value} where matching results are available.

We also compare against recent flow-based RL methods. \textsc{FQL} is an important baseline because it uses an expressive flow-matching policy together with a standard TD critic and a one-step policy trained to maximize Q-values while remaining regularized toward a BC flow policy; this helps isolate critic-side gains from improvements due purely to policy expressivity.
On the critic side, \textsc{floq} is a complementary comparison: rather than modelling a return distribution, it parametrises a scalar Q-function itself as a flow-matching model via a learned velocity field and iterative numerical integration. This lets us test whether the benefits of FlowIQN come specifically from learning the full return distribution, rather than from using a flow-based critic in a scalar-value setting.

Finally, we compare against \textsc{Value Flows}, which is the closest prior flow-based distributional critic baseline. Like FlowIQN, Value Flows models the full return distribution with a flow-matching critic and supports multiple policy extraction strategies, making it the most relevant prior comparison in spirit \cite{dongzheng2026value}. However, unlike FlowIQN, its critic objective is not shown to be Wasserstein-aligned, so it does not come with the same projection guarantee that motivates our method. Empirically, Value Flows reports very strong offline RL performance, so it is an important benchmark to beat. At the same time, its gains are harder to attribute purely to return-distribution modelling, since the method also estimates return variance to reweigh critic updates and includes an additional regularisation term that is itself shown to materially affect performance \cite{dongzheng2026value}. We therefore view Value Flows as both the strongest prior flow-based distributional baseline and a less clean critic-only ablation than the matched-actor comparisons with IQN and CODAC. 

\subsection{Critic network architecture}
We parameterize the scalar velocity field \(v_\theta(t, z_t, s, a, \tau)\) with an MLP. Following IQN \cite{pmlr-v80-dabney18a}, the quantile fraction \(\tau\) is embedded using a cosine basis and projected through a learned embedding network. Following floq \cite{agrawalla2026floq}, the current position in value space \(z_t\) is encoded with an HL-Gauss histogram embedding, and flow time \(t\) is encoded with a Fourier-feature MLP. We also learn a state-action embedding from \([s,a]\), combine it multiplicatively with the \(\tau\)-embedding as in IQN, and concatenate the result with the \(z_t\) and \(t\) embeddings before predicting the scalar velocity. For the shortcut variant, the step size \(d\) is embedded with the same Fourier construction and concatenated in the same way. Full architectural hyperparameters are given in Tables \ref{tab:flowiqn_hparams_common} and \ref{tab:flowiqn_hparams_tuned}.

\begin{table}[t]
\centering
\caption{Common hyperparameters for FlowIQN. We explicitly indicate which settings are inherited from prior flow-based offline RL pipelines and which are specific to FlowIQN.}
\label{tab:flowiqn_hparams_common}
\scriptsize 
\setlength{\tabcolsep}{3pt}
\renewcommand{\arraystretch}{0.95}

\begin{adjustbox}{max width=\textwidth}
\begin{tabular}{ll}
\toprule
\textbf{Hyperparameter} & \textbf{Value (FlowIQN)} \\
\midrule
Learning rate & $3\times10^{-3}$ \sameboth \\
Optimizer & Adam \sameboth \\
Gradient steps & 2M \sameboth \\
Minibatch size & 256 (default), 512 for hm-large, antmaze-giant \sameboth \\
Target network smoothing coeff. & 0.005 \sameboth \\
Discount factor $\gamma$ & 0.99 (default), 0.995 for antmaze-giant, humanoidmaze, antsoccer \sameboth \\
Flow time sampling distribution & $\mathrm{Unif}([0,1])$ \sameboth \\
Actor flow steps & 10 \sameboth \\
Clipped double Q-learning & False (mean over Q ensemble), True for antmaze-giant (min over Q ensemble) \samefql \\
BC / distillation coefficient $\alpha$ & Tables~\ref{tab:flowiqn_hparams_tuned} \\
\midrule
Critic MLP dims & [512,512,512,512] (default), [512,512] for FQL + cube envs \samefloq \\
MLP nonlinearity & GELU \sameboth \\
Distilled critic MLP dims & [512,512,512,512] \sameboth \\
Number of critic ensembles & 2 \samefql \\
\midrule
Default target quantile samples $K$ & 16 (default), Table~\ref{tab:flowiqn_hparams_tuned} for env-wise \ourspec \\
Critic flow steps $M$ & 8 (default), Table~\ref{tab:flowiqn_hparams_tuned} for env-wise \ourspec \\
Initial sample range $\kappa$ & 0.1 (default), Table~\ref{tab:flowiqn_hparams_tuned} for env-wise \ourspec \\
\midrule
Tau cosine basis size & 64 \ourspec \\
Shared $(s,a)$/$\tau$ embedding dim & 512 \ourspec \\
HL-Gauss number of bins & 51 \samefloq \\
HL-Gauss smoothing $\sigma$ & 16.0 \samefloq \\
HL-Gauss support $[v_{\min}, v_{\max}]$ & task-dependent \samefloq \\
Fourier time embedding dim & 128 \ourspec \\
Number of Fourier time frequencies & 64 \ourspec \\
\midrule
Shortcut step embedding dim & 128 \shortcut \\
Shortcut step sizes for consistency loss & [1/2, 1/4, 1/8] \shortcut \\
Shortcut consistency ratio $\lambda_c$ & 0.25 \shortcut \\
Weight decay & $5\times10^{-3}$ \shortcut \\
\bottomrule
\end{tabular}
\end{adjustbox}
\end{table}

\subsection{Evaluation protocol and hyperparameter search}
We follow the standard OGBench/FQL evaluation protocol used in recent work on this benchmark. In particular, following OGBench practice, hyperparameters are tuned on the default task of each environment (the \text{env-name-play-singletask-task*-v0} variant) and then reused for the remaining four subtasks in that environment. Default tasks from OGBench \cite{ogbench_park2025} are indicated by ($*$) in Table~\ref{tab:full_results_v2}. For hyperparameter sweeps, we use 3 seeds per candidate configuration. For final evaluation on each subtask, we use 8 fresh seeds and report the mean performance $\pm$ standard deviation across eval seeds in Table~\ref{tab:full_results_v2}. 

Our search procedure follows a similiar staged approach as floq \cite{agrawalla2026floq}. For our FQL-based policy extraction approach, \textsc{FlowIQN-FQL}, we begin from a shared default configuration (Table~\ref{tab:flowiqn_hparams_common}), first sweep the behavioural regularisation coefficient \(\alpha\) around $\pm\Delta$ their default FQL \cite{fql_park2025} value, with $\Delta=100$ for cube, puzzle and scene evironments, $\Delta=10$ for the humanoidmaze environment, and $\Delta=5$ for ant environments. We then tune the number of target quantile samples \(K\). Holding these fixed, we perform a small grid search over the flow-related hyperparameters: the source scaling parameter \(\kappa=\{0.1, 0.25\}\) and the number of critic flow steps \(M=\{4,8\}\). For experiments using rejection sampling, sweeping \(\alpha\) is unnecessary, we additionally sweep the number of candidate actions used at selection time \(J=\{8,16,32\}\) after performing the initial sweep stages outlined for \textsc{FlowIQN-FQL}. Shortcut variants reuse these values, with the consistency ratio fixed to $\lambda_c=0.25$ (Table~\ref{tab:flowiqn_hparams_common}). Final hyperparameter values can be found in Table~\ref{tab:flowiqn_hparams_tuned}.

Reproducing every prior baseline across the full benchmark suite is beyond our computational budget, so we limit reevaluations to \textsc{IQN} and \textsc{Value Flows}, using the reported hyperparameters and codebase from \cite{dongzheng2026value}. When other baselines results are available under the same benchmark/task definitions and evaluation protocol, we re-report numbers from these papers and mark them explicitly in all tables. We distinguish these reported results from our own runs throughout.

\begin{table*}[ht]
\vspace{-5pt}
\caption{Environment-specific hyperparameters used for FlowIQN variants. 
Hyperparameters are selected at the environment level and reused across all five subtasks. The first column reports the policy-extraction hyperparameter: actor regularization $\alpha$ for FlowIQN and the number of candidate actions $J$ for FlowIQN-R.}
\begin{adjustwidth}{-0.6in}{-0.6in}
\centering
\label{tab:flowiqn_hparams_tuned}
\scriptsize 
\setlength{\tabcolsep}{3pt}
\renewcommand{\arraystretch}{0.95}

\begin{adjustbox}{max width=\textwidth}
\begin{tabular}{ll cccc}
\toprule
\textbf{Method} & \textbf{Environment} 
& \textbf{Extraction hyperparam.}
& \textbf{Quantiles $K$} 
& \textbf{Source scale $\kappa$} 
& \textbf{Flow steps $M$} \\
\midrule

\multirow{8}{*}{FlowIQN-FQL}
& cube-double      & $\alpha=\text{300}$ & \text{8}  & \text{0.1}  & \text{8} \\
& cube-triple      & $\alpha=\text{300}$ & \text{8}  & \text{0.1}  & \text{8} \\
& puzzle-3x3       & $\alpha=\text{1000}$& \text{16} & \text{0.25} & \text{4} \\
& puzzle-4x4       & $\alpha=\text{900}$ & \text{16} & \text{0.1}  & \text{8} \\
& scene            & $\alpha=\text{200}$ & \text{8}  & \text{0.1}  & \text{8} \\
& antmaze-giant    & $\alpha=\text{15}$  & \text{16} & \text{0.1}  & \text{8} \\
& hmmaze-large     & $\alpha=\text{30}$  & \text{16} & \text{0.25} & \text{8} \\
& antsoccer-arena  & $\alpha=\text{15}$  & \text{16} & \text{0.25} & \text{8} \\
\midrule

\multirow{8}{*}{FlowIQN-R} 
& cube-double      & $J=\text{8}$ & \text{8}  & \text{0.1}  & \text{8} \\
& cube-triple      & $J=\text{8}$ & \text{8}  & \text{0.25} & \text{8} \\
& puzzle-3x3       & $J=\text{8}$ & \text{8}  & \text{0.1}  & \text{8} \\
& puzzle-4x4       & $J=\text{8}$ & \text{16} & \text{0.1}  & \text{8} \\
& scene            & $J=\text{8}$ & \text{16} & \text{0.25} & \text{8} \\
& antmaze-giant    & $J=\text{8}$ & \text{16} & \text{0.1}  & \text{8} \\
& hmmaze-large     & $J=\text{16}$& \text{16} & \text{0.25} & \text{8} \\
& antsoccer-arena  & $J=\text{8}$ & \text{16} & \text{0.25} & \text{8} \\

\bottomrule
\end{tabular}
\end{adjustbox}
\end{adjustwidth}
\vspace{-10pt}
\end{table*}

\section{Additional Results}
\label{appendix:results}

\subsection{Ablations on flow integrations effect on critic accuracy}
\label{appendix:flow_integration_ablations}

\begin{figure}[h]
    \centering

    \begin{minipage}[t]{0.48\columnwidth}
        \centering
        \includegraphics[width=\linewidth]{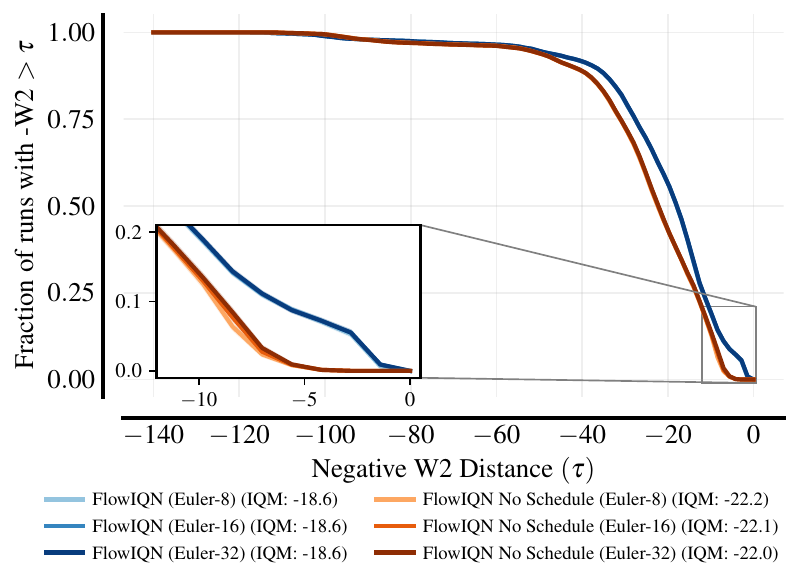}
        \vspace{-5pt}
        \centerline{\small (a) Adaptive schedule}
    \end{minipage}
    \hfill
    \begin{minipage}[t]{0.48\columnwidth}
        \centering
        \includegraphics[width=\linewidth]{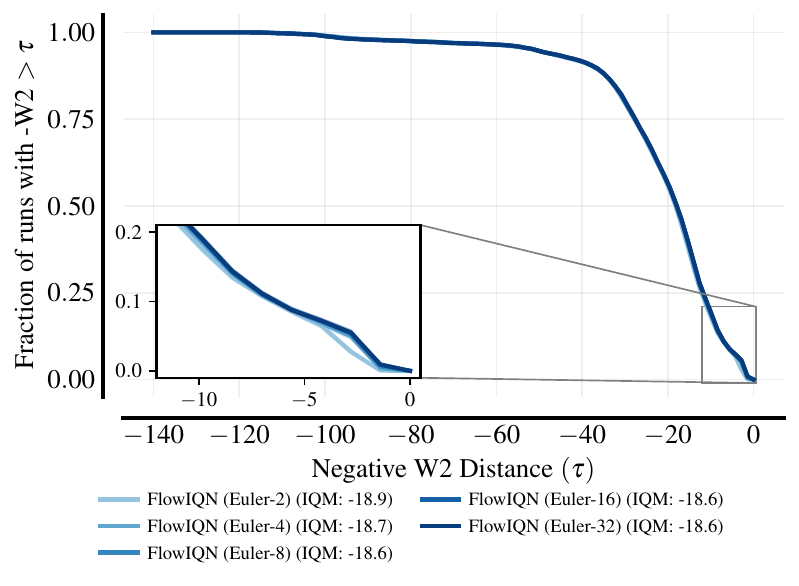}
        \vspace{-5pt}
        \centerline{\small (b) Euler steps}
    \end{minipage}

    \caption{
    \textbf{Ablations on FlowIQN integration.}
    Performance profiles of negative empirical $W_2$ distance in the fixed-policy evaluation setting; higher curves indicate lower Wasserstein error. 
    \textbf{(a)} The adaptive schedule improves accuracy at matched Euler step counts, while increasing the number of uniform steps alone does not recover the same performance.
    \textbf{(b)} Increasing the number of Euler steps yields only marginal gains beyond a small solver budget, with IQM scores saturating by $8$ steps.
    }
    \label{fig:w2_integration_ablations}
\end{figure}

We ablate the numerical integration choices used by FlowIQN in the fixed-policy return distribution modelling experiment. 
Figure~\ref{fig:w2_integration_ablations}a compares the default adaptive time schedule against a uniform Euler grid. We train FlowIQN with and without the adaptive time schedule described Appendix \ref{appendix:time_schedule} for 5 separate seeds. The adaptive schedule consistently improves $W_2$ accuracy at the same number of velocity evaluations, while increasing the number of uniform steps from $8$ to $32$ does not close the gap. This supports the use of an amortized schedule as a lightweight way to allocate solver effort to the higher-curvature regions of the learned return flow, rather than simply increasing the number of uniformly spaced integration steps.

Figure~\ref{fig:w2_integration_ablations}b varies the number of Euler steps used to evaluate the learned critic. The empirical $W_2$ profile changes only marginally beyond a small number of steps: the IQM improves from $-18.9$ with two Euler steps to $-18.6$ with eight or more steps. This suggests that, in this setting, FlowIQN's return samples are not strongly limited by solver resolution once a modest integration budget is used.

\subsection{Full Offline RL Results}
In this appendix, we present full evaluation results of both of our offline \ac{RL} method instantiations: \textsc{FlowIQN-FQL}, with FQL-style \cite{fql_park2025} policy extraction; and \textsc{FlowIQN-R}, with a rejection sampling based policy rule. 

\begin{table}[htbp]
\centering
\caption{Full offline RL results across OGBench tasks. We report average task success $\pm$ standard deviation, for each individual subtasks per environment. We denote values at or above $95\%$ the best performance in bold, following \cite{ogbench_park2025, fql_park2025}. Default tasks are indicated with (*).}
\label{tab:full_results_v2}
\scriptsize 
\setlength{\tabcolsep}{3pt}
\renewcommand{\arraystretch}{0.95}

\begin{adjustbox}{max width=\textwidth}
\begin{tabular}{l cc ccccccc} %
\toprule
 & \multicolumn{2}{c}{\textbf{Gaussian Policies}} & \multicolumn{7}{c}{\textbf{Flow Policies}} \\
\cmidrule(lr){2-3} \cmidrule(lr){4-10} %
\textbf{Task} & BC \cite{agrawalla2026floq} & ReBRAC \cite{agrawalla2026floq, dongzheng2026value} & IQN & CODAC \cite{dongzheng2026value} & FQL \cite{agrawalla2026floq} & floq \cite{agrawalla2026floq} & Value Flows \cite{dongzheng2026value} & FlowIQN-FQL & FlowIQN-R \\
\midrule

cube-double-play-singletask-task1-v0       & $8\pm3$  & $45\pm6$ & $\mathbf{92\pm4}$ & $80\pm11$ & $61\pm9$ & $50\pm24$ & $\mathbf{93\pm5}$ & $66\pm11$ & $85\pm12$ \\
cube-double-play-singletask-task2-v0 (*)   & $0\pm0$  & $7\pm3$  & $68\pm6$ & $63\pm4$ & $36\pm6$ & $\mathbf{72\pm15}$ & $\mathbf{74\pm5}$ & $63\pm16$ &  $\mathbf{75\pm9}$ \\
cube-double-play-singletask-task3-v0       & $0\pm0$  & $4\pm1$  & $58\pm7$ & $66\pm9$ & $22\pm5$ & $57\pm14$ & $64\pm6$ & $64\pm19$ & $\mathbf{83\pm6}$ \\
cube-double-play-singletask-task4-v0       & $0\pm0$  & $1\pm1$  & $21\pm5$ & $13\pm2$ & $5\pm2$  & $8\pm4$  & $20\pm6$ & $17\pm6$ & $\mathbf{26\pm10}$ \\
cube-double-play-singletask-task5-v0       & $0\pm0$  & $4\pm2$  & $83\pm3$ & $82\pm4$ & $19\pm10$ & $50\pm11$ & $62\pm8$ & $52\pm24$ & $\mathbf{90\pm5}$ \\
\midrule

cube-triple-play-singletask-task1-v0       & $1\pm1$  & $1\pm2$  & $71\pm10$ & $9\pm5$  & $20\pm6$ & -        & $\mathbf{76\pm12}$ & $38\pm20$ & $15\pm18$ \\
cube-triple-play-singletask-task2-v0 (*)   & $0\pm0$  & $0\pm0$  & $\mathbf{17\pm7}$  & $1\pm0$  & $1\pm2$  & -        & $0\pm0$  & $1\pm1$ & $1\pm1$ \\
cube-triple-play-singletask-task3-v0       & $0\pm0$  & $0\pm0$  & $\mathbf{22\pm6}$  & $0\pm0$  & $0\pm0$  & -        & $17\pm8$  & $8\pm8$ & $4\pm5$ \\
cube-triple-play-singletask-task4-v0       & $0\pm0$  & $0\pm0$  & $\mathbf{3\pm2}$   & $0\pm0$   & $0\pm0$   & -        & $1\pm2$   & $0\pm0$ & $1\pm1$ \\
cube-triple-play-singletask-task5-v0       & $0\pm0$  & $0\pm0$  & $6\pm3$  & $0\pm0$  & $0\pm0$  & -        & $\mathbf{7\pm5}$  & $0\pm0$ & $4\pm6$ \\
\midrule

puzzle-3x3-play-singletask-task1-v0        & $5\pm2$  & $\mathbf{97\pm4}$ & $\mathbf{96\pm2}$ & $78\pm8$ & $90\pm4$ & $\mathbf{95\pm4}$ & $\mathbf{95\pm6}$ & $\mathbf{95\pm2}$ & $92\pm9$ \\
puzzle-3x3-play-singletask-task2-v0        & $1\pm1$  & $1\pm1$  & $27\pm4$  & $5\pm2$  & $16\pm5$ & $18\pm10$ & $\mathbf{34\pm29}$ & $20\pm8$ & $25\pm14$ \\
puzzle-3x3-play-singletask-task3-v0        & $1\pm1$  & $3\pm1$  & $15\pm4$  & $4\pm3$  & $10\pm3$ & $16\pm7$ & $\mathbf{29\pm32}$ & $21\pm5$ & $17\pm16$ \\
puzzle-3x3-play-singletask-task4-v0 (*)    & $1\pm1$  & $2\pm1$  & $2\pm1$  & $5\pm5$  & $16\pm5$ & $17\pm6$ & $\mathbf{25\pm23}$ & $23\pm7$ & $5\pm6$ \\
puzzle-3x3-play-singletask-task5-v0        & $1\pm0$  & $5\pm3$  & $11\pm4$  & $6\pm5$  & $16\pm3$ & $\mathbf{38\pm6}$ & $21\pm26$ & $28\pm7$ & $10\pm5$ \\

\midrule

puzzle-4x4-play-singletask-task1-v0        & $1\pm1$  & $26\pm4$ & $\mathbf{60\pm10}$ & $37\pm32$ & $34\pm8$ & $47\pm7$ & $31\pm9$ & $32\pm8$ & $51\pm14$ \\
puzzle-4x4-play-singletask-task2-v0        & $0\pm0$  & $12\pm4$ & $12\pm3$ & $10\pm10$ & $16\pm5$ & $21\pm6$ & $28\pm8$ & $9\pm2$  & $\mathbf{56\pm20}$ \\
puzzle-4x4-play-singletask-task3-v0        & $0\pm0$  & $15\pm3$ & $43\pm4$ & $33\pm29$ & $18\pm5$ & $36\pm5$ & $27\pm4$ & $18\pm3$  & $\mathbf{48\pm21}$ \\
puzzle-4x4-play-singletask-task4-v0 (*)    & $0\pm0$  & $10\pm3$ & $21\pm4$ & $12\pm10$ & $11\pm3$ & $19\pm5$ & $20\pm3$ & $11\pm4$ & $\mathbf{36\pm23}$ \\
puzzle-4x4-play-singletask-task5-v0        & $0\pm0$  & $7\pm3$  & $11\pm7$ & $10\pm8$  & $7\pm3$  & $16\pm7$ & $16\pm6$ & $13\pm2$ & $\mathbf{46\pm29}$ \\
\midrule

scene-play-singletask-task1-v0             & $19\pm6$ & $\mathbf{95\pm2}$ & $\mathbf{100\pm0}$ & $\mathbf{99\pm0}$ & $\mathbf{100\pm0}$ & $\mathbf{100\pm1}$ & $\mathbf{99\pm1}$ & $\mathbf{100\pm0}$ & $\mathbf{100\pm0}$\\
scene-play-singletask-task2-v0 (*)         & $1\pm1$  & $50\pm13$ & $\mathbf{100\pm0}$  & $85\pm4$ & $76\pm9$ & $83\pm10$ & $92\pm11$ & $86\pm8$ & $\mathbf{99\pm1}$ \\
scene-play-singletask-task3-v0             & $1\pm1$  & $55\pm16$ & $\mathbf{97\pm3}$ & $90\pm3$ & $\mathbf{98\pm1}$ & $\mathbf{98\pm2}$ & $92\pm3$ & $93\pm4$ & $\mathbf{95\pm3}$ \\
scene-play-singletask-task4-v0             & $2\pm2$  & $3\pm3$  & $0\pm0$  & $0\pm0$  & $5\pm1$  & $\mathbf{9\pm7}$   & $2\pm3$ & $1\pm2$  & $0\pm0$ \\
scene-play-singletask-task5-v0             & $\mathbf{0\pm0}$  & $\mathbf{0\pm0}$  & $\mathbf{0\pm0}$  & $\mathbf{0\pm0}$  & $\mathbf{0\pm0}$  & $\mathbf{0\pm0}$   & $\mathbf{0\pm0}$  & $\mathbf{0\pm0}$  &  $\mathbf{0\pm0}$ \\
\midrule

antmaze-giant-navigate-singletask-task1-v0 (*)   & $0\pm0$ & $27\pm22$ & - & - & $11\pm16$ & $\mathbf{86\pm4}$ & - & $22\pm13$ & $5\pm5$ \\
antmaze-giant-navigate-singletask-task2-v0    & $0\pm0$ & $16\pm17$ & - & - & $\mathbf{68\pm17}$ & $\mathbf{66\pm11}$ & - & $9\pm14$ & $43\pm28$ \\
antmaze-giant-navigate-singletask-task3-v0    & $0\pm0$ & $\mathbf{34\pm22}$ & - & - & $3\pm4$ & $0\pm0$ & - & $0\pm0$ & $0\pm0$ \\
antmaze-giant-navigate-singletask-task4-v0    & $0\pm0$ & $5\pm12$ & - & - & $31\pm36$ & $17\pm23$ & - & $\mathbf{44\pm28}$ & $31\pm16$ \\
antmaze-giant-navigate-singletask-task5-v0    & $1\pm1$ & $49\pm22$ & - & - & $21\pm31$ & $\mathbf{84\pm8}$ & - & $35\pm34$ & $44\pm7$ \\

\midrule

humanoidmaze-large-navigate-singletask-task1-v0 (*)  & $0\pm0$ & $2\pm1$ & - & - & $14\pm10$ & $\mathbf{52\pm8}$ & - & $24\pm11$ & $29\pm25$ \\
humanoidmaze-large-navigate-singletask-task2-v0   & $\mathbf{0\pm0}$ & $\mathbf{0\pm0}$ & - & - & $\mathbf{0\pm0}$ & $\mathbf{0\pm1}$ & - & $\mathbf{0\pm0}$ & $\mathbf{0\pm0}$ \\
humanoidmaze-large-navigate-singletask-task3-v0   & $1\pm1$ & $8\pm4$ & - & - & $18\pm4$ & $28\pm12$ & - & $\mathbf{31\pm12}$ & $9\pm7$ \\
humanoidmaze-large-navigate-singletask-task4-v0   & $1\pm0$ & $1\pm1$ & - & - & $\mathbf{24\pm7}$ & $22\pm12$ & - & $16\pm5$ & $0\pm0$ \\
humanoidmaze-large-navigate-singletask-task5-v0   & $0\pm1$ & $2\pm2$ & - & - & $22\pm9$ & $38\pm9$ & - & $16\pm3$ & $\mathbf{54\pm7}$ \\

\midrule

antsoccer-arena-navigate-singletask-task1-v0    & $2\pm1$ & $0\pm0$ & - & - & $78\pm7$ & $\mathbf{92\pm4}$ & - & $76\pm6$  & $\mathbf{92\pm3}$ \\
antsoccer-arena-navigate-singletask-task2-v0    & $2\pm2$ & $0\pm0$ & - & - & $\mathbf{89\pm4}$ & $\mathbf{92\pm2}$ & - & $71\pm5$  & $\mathbf{91\pm3}$ \\
antsoccer-arena-navigate-singletask-task3-v0    & $0\pm0$ & $0\pm0$ & - & - & $\mathbf{61\pm6}$ & $\mathbf{58\pm11}$ & - & $23\pm9$  & $50\pm7$ \\
antsoccer-arena-navigate-singletask-task4-v0 (*)    & $1\pm0$ & $0\pm0$ & - & - & $\mathbf{49\pm11}$ & $\mathbf{49\pm10}$ & - & $41\pm11$  & $\mathbf{48\pm5}$ \\
antsoccer-arena-navigate-singletask-task5-v0    & $0\pm0$ & $0\pm0$ & - & - & $\mathbf{38\pm5}$ & $32\pm23$ & - & $10\pm6$  & $24\pm3$ \\

\midrule
\end{tabular}
\end{adjustbox}
\end{table}

\paragraph{Shortcut results.}
Results for shortcut variants of \textsc{FlowIQN-FQL} and \textsc{FlowIQN-R} are presented in Table~\ref{tab:full_shortcut_results_v2}. The shortcut variants should be interpreted primarily as an efficiency-oriented extension rather than as a uniformly stronger offline RL method. In most environments, shortcut critics underperform their corresponding multi-step FlowIQN counterparts, suggesting that the reduced inference cost can come with an optimisation or capacity trade-off in downstream control. The main exception is the \texttt{cube-double} tasks, where FlowIQN-R-S achieves the strongest result among all evaluated methods, including the non-shortcut FlowIQN variants and prior baselines. For this environment, we follow \citet{agrawalla2026floq} and use smaller critic network sizes, motivated by their observation that larger critics may overfit on this task. This likely contributes to the strong shortcut result, but we do not interpret it as evidence that shortcuts are generally superior. Overall, the shortcut experiments support the view that shortcut critics can preserve practical utility while reducing inference cost, but their benefits depend on environment and model-size choices.

\begin{table}[htbp]
\centering
\caption{Comparison against shortcut model variants (\textsc{-S}), of our methods, \textsc{FlowIQN-FQL} and \textsc{FlowIQN-R}. We report average task success $\pm$ standard deviation, for each individual subtasks per environment. We denote values at or above $95\%$ the best performance in bold, following \cite{ogbench_park2025, fql_park2025}. Default tasks are indicated with (*).}
\label{tab:full_shortcut_results_v2}
\scriptsize 
\setlength{\tabcolsep}{3pt}
\renewcommand{\arraystretch}{0.95}

\begin{adjustbox}{max width=\textwidth}
\begin{tabular}{l cccc} %
\toprule
 & \multicolumn{4}{c}{\textbf{Flow Policies}} \\
 \cmidrule(lr){2-5} %
\textbf{Task} & FlowIQN-FQL & FlowIQN-FQL-S & FlowIQN-R & FlowIQN-R-S \\
\midrule

cube-double-play-singletask-task1-v0    & $66\pm11$ & $78\pm11$ & $85\pm12$ & $\mathbf{93\pm6}$ \\
cube-double-play-singletask-task2-v0 (*)& $63\pm16$ & $68\pm7$ & $75\pm9$ & $\mathbf{88\pm5}$ \\
cube-double-play-singletask-task3-v0    & $64\pm19$ & $34\pm17$ & $83\pm6$ & $\mathbf{91\pm5}$ \\
cube-double-play-singletask-task4-v0    & $17\pm6$ & $16\pm2$ & $26\pm10$ & $\mathbf{38\pm7}$ \\
cube-double-play-singletask-task5-v0    & $52\pm24$  & $40\pm15$ & $\mathbf{90\pm5}$  & $\mathbf{86\pm10}$ \\

\midrule

cube-triple-play-singletask-task1-v0    & $\mathbf{38\pm20}$ & $20\pm13$ & $15\pm18$ & $24\pm6$ \\
cube-triple-play-singletask-task2-v0 (*)& $1\pm1$ & $1\pm1$ & $1\pm1$ & $\mathbf{3\pm3}$ \\
cube-triple-play-singletask-task3-v0    & $\mathbf{8\pm8}$ & $1\pm1$ & $4\pm5$ & $0\pm1$ \\
cube-triple-play-singletask-task4-v0    & $0\pm0$ & $0\pm1$ & $\mathbf{1\pm1}$ & $0\pm0$ \\
cube-triple-play-singletask-task5-v0    & $0\pm0$ & $0\pm0$ & $\mathbf{4\pm6}$ & $0\pm0$ \\

\midrule

puzzle-3x3-play-singletask-task1-v0     & $\mathbf{95\pm2}$ & $70\pm11$ & $\mathbf{92\pm9}$ & $86\pm5$ \\
puzzle-3x3-play-singletask-task1-v0     & $20\pm8$ & $5\pm3$ & $\mathbf{25\pm14}$ & $2\pm2$ \\
puzzle-3x3-play-singletask-task1-v0     & $\mathbf{21\pm5}$ & $4\pm3$ & $17\pm16$ & $2\pm1$ \\
puzzle-3x3-play-singletask-task1-v0 (*) & $\mathbf{23\pm7}$ & $8\pm4$ & $5\pm6$ & $1\pm2$ \\
puzzle-3x3-play-singletask-task1-v0     & $\mathbf{28\pm7}$ & $5\pm2$ & $10\pm5$ & $1\pm2$ \\

\midrule

puzzle-4x4-play-singletask-task1-v0     & $32\pm8$ & $28\pm11$ & $\mathbf{51\pm14}$ & $47\pm11$ \\
puzzle-4x4-play-singletask-task2-v0     & $9\pm2$ & $6\pm2$ & $\mathbf{56\pm20}$ & $18\pm10$ \\
puzzle-4x4-play-singletask-task3-v0     & $18\pm3$ & $18\pm4$ & $\mathbf{48\pm21}$ & $29\pm10$ \\
puzzle-4x4-play-singletask-task4-v0 (*) & $11\pm4$ & $8\pm3$ & $\mathbf{36\pm23}$ & $19\pm6$ \\
puzzle-4x4-play-singletask-task5-v0     & $13\pm2$ & $8\pm4$ & $\mathbf{46\pm29}$ & $24\pm6$ \\

\midrule

scene-play-singletask-task1-v0          & $\mathbf{100\pm0}$ & $\mathbf{99\pm4}$ & $\mathbf{100\pm0}$ & $\mathbf{99\pm1}$ \\
scene-play-singletask-task1-v0 (*)      & $86\pm8$ & $42\pm20$ & $\mathbf{99\pm1}$ & $8\pm6$ \\
scene-play-singletask-task1-v0          & $\mathbf{93\pm4}$ & $85\pm4$ & $\mathbf{95\pm3}$ & $81\pm8$ \\
scene-play-singletask-task1-v0          & $\mathbf{1\pm2}$ & $\mathbf{1\pm1}$ & $0\pm0$ & $0\pm0$ \\
scene-play-singletask-task1-v0          & $\mathbf{0\pm0}$ & $0\pm0$ & $\mathbf{0\pm0}$ & $0\pm0$ \\

\midrule

antmaze-giant-navigate-singletask-task1-v0 (*)& $\mathbf{22\pm13}$ & $0\pm0$ & $5\pm5$ & $0\pm0$ \\
antmaze-giant-navigate-singletask-task2-v0    & $9\pm14$ & $1\pm1$ & $\mathbf{43\pm28}$ & $0\pm0$ \\
antmaze-giant-navigate-singletask-task3-v0    & $\mathbf{0\pm0}$ & $\mathbf{0\pm0}$ & $\mathbf{0\pm0}$ & $\mathbf{0\pm0}$ \\
antmaze-giant-navigate-singletask-task4-v0    & $44\pm28$ & $0\pm0$ & $31\pm16$ & $0\pm0$ \\
antmaze-giant-navigate-singletask-task5-v0    & $35\pm34$ & $0\pm0$ & $\mathbf{44\pm7}$ & $0\pm0$\\

\midrule

humanoidmaze-large-navigate-singletask-task1-v0 (*) & $24\pm11$ & $4\pm3$ & $\mathbf{29\pm25}$ & $17\pm7$ \\
humanoidmaze-large-navigate-singletask-task2-v0   & $\mathbf{0\pm0}$ & $\mathbf{0\pm0}$ & $\mathbf{0\pm0}$ & $\mathbf{0\pm0}$ \\
humanoidmaze-large-navigate-singletask-task3-v0   & $\mathbf{31\pm12}$ & $10\pm3$ & $9\pm7$ & $8\pm4$ \\
humanoidmaze-large-navigate-singletask-task4-v0   & $\mathbf{16\pm5}$ & $2\pm2$ & $0\pm0$ & $0\pm0$ \\
humanoidmaze-large-navigate-singletask-task5-v0   & $16\pm3$ & $1\pm1$ & $\mathbf{54\pm7}$ & $0\pm0$ \\

\midrule

antsoccer-arena-navigate-singletask-task1-v0    & $76\pm6$ & $53\pm8$ & $\mathbf{92\pm3}$ & $82\pm2$ \\
antsoccer-arena-navigate-singletask-task2-v0    & $71\pm5$ & $60\pm5$ & $\mathbf{91\pm3}$ & $73\pm8$ \\
antsoccer-arena-navigate-singletask-task3-v0    & $23\pm9$ & $22\pm8$ & $\mathbf{50\pm7}$ & $40\pm10$ \\
antsoccer-arena-navigate-singletask-task4-v0 (*)& $41\pm11$ & $20\pm6$ & $\mathbf{48\pm5}$ & $33\pm6$ \\
antsoccer-arena-navigate-singletask-task5-v0   & $10\pm6$ & $16\pm9$ & $\mathbf{24\pm3}$ & $9\pm6$ \\

\midrule
\end{tabular}
\end{adjustbox}
\end{table}

\clearpage
\subsection{Isolating the quantile coupling}
\label{app:coupling_ablation}

This section gives details and full results for Section~\ref{sec:coupling_ablation}. In all three tasks, FlowIQN and the uncoupled flow share the network (an MLP with two hidden layers of 128 units), initialisation, mini-batches, $K=32$ target samples per group, interpolant and $L^2$ regression loss. FlowIQN sorts the source values $z_0=g(\tau)$, the quantile fractions $\tau$ and the target samples $y$ within each group; the uncoupled flow leaves all three unsorted, which pairs source and target samples independently. For each critic, let $F^{-1}_{\eta_\theta}(\tau\mid s,a)$ denote its learned quantile function, which for FlowIQN is $Q_\theta^{\mathrm{flow}}(s,a,g(\tau),\tau)$ (Section~\ref{sec:flowiqn}). Targets are re-sampled Monte Carlo draws, the regime in which the sorted mini-batch coupling approximates the population quantile coupling of Section~\ref{sec:wa_projection}. IQN and C51 are trained on the same targets and serve as references.

\textbf{Accuracy.}
Each target return distribution $\nu(\cdot\mid s,a)$ is a Gaussian mixture $\sum_{j=1}^M w_j\,\mathcal N(\mu_j,\sigma^2)$ with $\sigma=0.5$ and equal weights $w_j=1/M$, standing in for a state-action pair whose return branches into $M$ modes. Critics are trained for 5,000 steps by direct regression on Monte Carlo draws from $\nu$, without bootstrapping, so the number of Euler steps is an inference budget only. Table~\ref{tab:app_coupling_accuracy} reports $W_2(\eta_\theta(\cdot\mid s,a),\nu(\cdot\mid s,a))$ to the closed-form target, and Figure~\ref{fig:density_overlay} shows the recovered densities at 4 Euler steps.

\begin{table}[!htbp]
\caption{Accuracy task: $W_2$ to the closed-form target (mean over 5 seeds; lower is better). FlowIQN and the uncoupled flow are evaluated at 2, 8 and 32 Euler steps. Best result in each row in \textbf{bold}; best flow critic \underline{underlined}.}
\label{tab:app_coupling_accuracy}
\centering
\small
\setlength{\tabcolsep}{3.5pt}
\begin{tabular}{lcccccccc}
\toprule
& \multicolumn{3}{c}{FlowIQN} & \multicolumn{3}{c}{Uncoupled} & & \\
\cmidrule(lr){2-4}\cmidrule(lr){5-7}
Target (mode locations) & 2 & 8 & 32 & 2 & 8 & 32 & IQN & C51 \\
\midrule
Unimodal ($0$) & $\underline{0.08}$ & $0.09$ & $0.10$ & $0.26$ & $0.18$ & $0.14$ & $0.02$ & $\mathbf{0.01}$ \\
Bimodal ($\pm 3$) & $0.73$ & $0.72$ & $\underline{0.65}$ & $1.87$ & $1.75$ & $1.50$ & $\mathbf{0.17}$ & $0.23$ \\
Bimodal, close ($\pm 1.5$) & $\underline{0.23}$ & $0.26$ & $0.26$ & $1.02$ & $0.93$ & $0.79$ & $0.05$ & $\mathbf{0.04}$ \\
Trimodal ($0,\pm 4$) & $\underline{0.59}$ & $0.60$ & $0.65$ & $1.73$ & $1.54$ & $1.09$ & $\mathbf{0.13}$ & $0.15$ \\
Quadmodal ($\pm 1.5,\pm 4.5$) & $\underline{0.48}$ & $0.50$ & $0.54$ & $1.83$ & $1.66$ & $1.23$ & $\mathbf{0.09}$ & $\mathbf{0.09}$ \\
\bottomrule
\end{tabular}
\end{table}

\begin{figure}[!htbp]
\centering
\includegraphics[width=\linewidth]{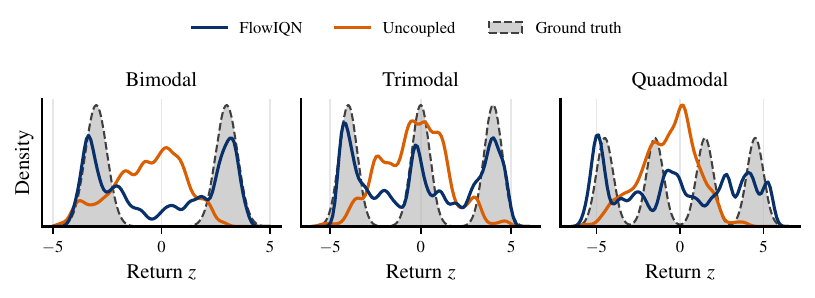}
\caption{Recovered return densities at 4 Euler steps (seed 0). FlowIQN preserves the modes of the ground-truth target distribution; the uncoupled flow places mass between them.}
\label{fig:density_overlay}
\end{figure}

\textbf{Decision.}
The task has 32 states $s$, each with two actions. The \emph{risky} action $a_{\mathrm{r}}$ has a bimodal return distribution $Z(s,a_{\mathrm{r}})\sim\tfrac12\mathcal N(m-d,\sigma^2)+\tfrac12\mathcal N(m+d,\sigma^2)$ with $d\in[1.5,2.5]$; the \emph{safe} action $a_{\mathrm{s}}$ has a unimodal return distribution $Z(s,a_{\mathrm{s}})\sim\mathcal N(m+\kappa d,\,0.25^2)$ with $\kappa\in[0.7,0.9]$, so it has the higher expected return. The policy is greedy in the $0.9$-quantile, $\pi(s)=\arg\max_a F^{-1}_{\eta_\theta}(0.9\mid s,a)$. Under the true return distributions this selects $a_{\mathrm{r}}$ in every state; a critic that does not resolve the high mode of $Z(s,a_{\mathrm{r}})$ selects $a_{\mathrm{s}}$ instead. In the companion task, $a_{\mathrm{s}}$ is given the same expected return as $a_{\mathrm{r}}$, so the decision depends on spread alone. Table~\ref{tab:app_coupling_decision} reports the fraction of states in which $\pi$ selects the optimal action, and the regret $F^{-1}_{Z(s,a^\star)}(0.9)-F^{-1}_{Z(s,\pi(s))}(0.9)$ in the true $0.9$-quantile, averaged over states.

\begin{table}[!htbp]
\caption{Decision task: fraction of states with the optimal action (higher is better) and regret in the true $0.9$-quantile (lower is better), at each number of Euler steps. Mean $\pm$ standard deviation over 5 seeds of 32 states each. IQN and C51 select the optimal action in every state of both tasks (regret $0$).}
\label{tab:app_coupling_decision}
\centering
\small
\setlength{\tabcolsep}{4pt}
\begin{tabular}{llccccc}
\toprule
& & \multicolumn{5}{c}{Euler steps} \\
\cmidrule(lr){3-7}
Metric & Critic & 2 & 4 & 8 & 16 & 32 \\
\midrule
\multicolumn{7}{l}{\emph{Main task}} \\
Optimal action $\uparrow$ & FlowIQN & $1.00$ & $1.00$ & $1.00$ & $1.00$ & $1.00$ \\
 & Uncoupled & $0.12\pm0.24$ & $0.15\pm0.28$ & $0.29\pm0.36$ & $0.44\pm0.46$ & $0.57\pm0.45$ \\
Regret $\downarrow$ & FlowIQN & $0$ & $0$ & $0$ & $0$ & $0$ \\
 & Uncoupled & $0.37\pm0.10$ & $0.36\pm0.12$ & $0.30\pm0.16$ & $0.24\pm0.20$ & $0.18\pm0.20$ \\
\midrule
\multicolumn{7}{l}{\emph{Companion task (equal expected returns)}} \\
Optimal action $\uparrow$ & FlowIQN & $1.00$ & $1.00$ & $1.00$ & $1.00$ & $1.00$ \\
 & Uncoupled & $1.00$ & $1.00$ & $0.99\pm0.01$ & $1.00$ & $1.00$ \\
\bottomrule
\end{tabular}
\end{table}

\textbf{Risk.}
We use the Machine-Replacement MDP with a CVaR$_{0.1}$ objective, following the risk-sensitive evaluation of Value Flows \cite{dongzheng2026value}. Because the MDP is small and discrete, the dataset contains many transitions per state-action pair, so Bellman targets can be re-sampled. For a policy $\pi$ with return $Z^\pi$ from the initial state, $\mathrm{CVaR}_{0.1}(Z^\pi)=\mathbb E\bigl[Z^\pi\mid Z^\pi\le F^{-1}_{Z^\pi}(0.1)\bigr]$ is the expected return in the worst $10\%$ of outcomes. The CVaR$_{0.1}$-optimal policy attains a return CVaR$_{0.1}$ of $-18.50$; the risk-neutral policy, which maximises the expected return, attains $-20.08$. Table~\ref{tab:app_coupling_risk} reports the number of seeds whose final greedy policy is CVaR$_{0.1}$-optimal, the true return CVaR$_{0.1}$ of that policy, and the tail error, the mean absolute error between $F^{-1}_{\eta_\theta}(\tau\mid s,a)$ and the true quantile function for $\tau\in(0,0.1]$ at the state where the decision is made. The learned policies are identical at every budget from 2 to 32 Euler steps, for both FlowIQN and the uncoupled flow; FlowIQN's one miss is the same seed at every budget.

\begin{table}[!htbp]
\caption{Risk task (6 seeds). The learned policy and its CVaR$_{0.1}$ are identical at every number of Euler steps from 2 to 32; for the flow critics, tail error is given as its range over these budgets. CVaR$_{0.1}$ is of the true return under the learned greedy policy (mean $\pm$ standard deviation); the optimum is $-18.50$. Best result in each column in \textbf{bold}; best flow critic underlined.}
\label{tab:app_coupling_risk}
\centering
\small
\begin{tabular}{lccc}
\toprule
Critic & Optimal policy (seeds) $\uparrow$ & Return CVaR$_{0.1}$ $\uparrow$ & Tail error $\downarrow$ \\
\midrule
FlowIQN & $\underline{5/6}$ & $\underline{-18.55\pm0.12}$ & $\underline{7.8\text{--}8.3}$ \\
Uncoupled & $0/6$ & $-20.08$ & $48.1\text{--}49.4$ \\
Value Flows & $\underline{5/6}$ & $-18.61\pm0.15$ & $8.6$ \\
IQN & $\mathbf{6/6}$ & $\mathbf{-18.50}$ & $5.3$ \\
C51 & $\mathbf{6/6}$ & $\mathbf{-18.50}$ & $\mathbf{1.3}$ \\
\bottomrule
\end{tabular}
\end{table}

\subsection{Sensitivity to the number of target quantile samples \texorpdfstring{$K$}{K}}
\label{app:k_sweep}

We vary the number of target samples $K$ sorted per transition on the fixed-policy diagnostic of Section~\ref{sec:policy_eval_exp}, keeping all other settings fixed (Table~\ref{tab:app_k_sweep}). Accuracy is flat for $K\ge4$, and $K=1$ and $K=2$ are no worse, though their confidence intervals are wide. No value of $K$ is reliably more accurate than IQN ($-16.7$): the $K=2$ point estimate is higher, but its interval contains IQN's IQM, and no value approaches C51 ($-12.2$). Note that $K=1$ is not standard \ac{CFM}: the sort is then trivial, but the source value $g(\tau)$ is still tied to the quantile level $\tau$ at which the target critic is evaluated (Section~\ref{sec:flowiqn}), whereas independently coupled \ac{CFM} pairs source and target samples at random for every $K$.

\begin{table}[!htbp]
\caption{IQM of $-W_2$ on the fixed-policy diagnostic for different $K$ (higher is better), with $95\%$ stratified-bootstrap confidence intervals. $K=8$ is the default and uses 5 seeds; all other values use 3 seeds.}
\label{tab:app_k_sweep}
\centering
\small
\begin{tabular}{rcc}
\toprule
$K$ & IQM & 95\% CI \\
\midrule
1 & $-17.0$ & $[-17.8, -15.6]$ \\
2 & $-15.9$ & $[-17.4, -14.1]$ \\
4 & $-18.4$ & $[-18.8, -17.8]$ \\
8 & $-18.6$ & $[-19.6, -17.7]$ \\
16 & $-18.6$ & $[-19.8, -18.1]$ \\
32 & $-18.7$ & $[-19.8, -17.4]$ \\
\bottomrule
\end{tabular}
\end{table}

\subsection{Per-seed fixed-policy results}
\label{app:per_seed}

The confidence intervals in Figure~\ref{fig:w2_performance_profiles} are narrow because each curve aggregates 1,578 state-action pairs. Table~\ref{tab:app_per_seed} reports the IQM of $-W_2$ for each seed separately. FlowIQN's per-seed IQMs vary more than those of the other critics (standard deviation $1.2$, against $0.3$--$0.6$): its best seed ($-16.9$) lies within the range of IQN's seeds ($-17.3$ to $-16.2$), and its worst ($-20.0$) is $3.3$ below IQN's mean and $3.6$ below CODAC's.

\begin{table}[!htbp]
\caption{Per-seed IQM of $-W_2$ on the fixed-policy diagnostic (higher is better).}
\label{tab:app_per_seed}
\centering
\small
\begin{tabular}{lcccccc}
\toprule
Critic & Seed 0 & Seed 1 & Seed 2 & Seed 3 & Seed 4 & Mean $\pm$ std \\
\midrule
C51 & $-11.8$ & $-13.1$ & $-11.8$ & $-12.4$ & $-11.8$ & $-12.2\pm0.6$ \\
IQN & $-16.9$ & $-16.5$ & $-17.3$ & $-16.8$ & $-16.2$ & $-16.7\pm0.4$ \\
CODAC & $-16.0$ & $-16.6$ & $-16.4$ & $-16.6$ & $-16.1$ & $-16.4\pm0.3$ \\
Value Flows & $-54.4$ & $-54.0$ & $-53.6$ & $-54.5$ & $-54.3$ & $-54.2\pm0.4$ \\
FlowIQN & $-20.0$ & $-16.9$ & $-18.6$ & $-18.1$ & $-19.5$ & $-18.6\pm1.2$ \\
\bottomrule
\end{tabular}
\end{table}

\subsection{Missing baseline cells}
\label{app:missing_cells}

Ten cells of Table~\ref{tab:aggregated_results_v2} are missing: IQN, CODAC and Value Flows on \texttt{antmaze-giant}, \texttt{humanoidmaze-large} and \texttt{antsoccer-arena}, and floq on \texttt{cube-triple}. CODAC results are re-reported from \cite{dongzheng2026value}, which does not evaluate these environments. IQN and Value Flows are the baselines we re-ran ourselves, using the code and hyperparameters of \cite{dongzheng2026value}, and we ran them on the five manipulation environments only. Filling the six re-run cells would require 240 additional runs under the full protocol (5 tasks $\times$ 8 seeds $\times$ 2M steps per run), on the order of $10^3$ GPU-hours on our hardware. The floq results are re-reported from \cite{agrawalla2026floq}, which does not report \texttt{cube-triple}.

\section{Broader Impacts}
\label{app:broader_impacts}

This work is primarily foundational research on distributional reinforcement learning and flow-based critic modelling. Its potential positive impacts are indirect: improved return-distribution modelling may support more reliable uncertainty-aware and risk-sensitive decision-making methods, which could be useful in robotics and other sequential decision-making settings where understanding the range of possible outcomes is important.

At the same time, improved reinforcement learning methods can have negative impacts if deployed in safety-critical settings without adequate validation. In particular, more expressive critics may encourage overconfidence in learned policies if distributional estimates are treated as reliable outside the support of the training data. This is especially relevant in offline RL, where learned policies may select actions that are poorly represented in the dataset. Our experiments are limited to benchmark simulation environments, and the method is not evaluated for real-world deployment. Any application to physical systems or high-stakes decision-making would require additional safety analysis, uncertainty calibration, out-of-distribution detection, and domain-specific validation.

\section{Compute Resources}
\label{app:compute}

The experiments were run on heterogeneous compute resources. Development and debugging were performed on a local workstation with an NVIDIA RTX 5090 GPU. Some hyperparameter tuning and evaluation runs were performed on an institutional high-throughput GPU cluster, primarily using NVIDIA L40S GPUs. The majority of hyperparameter tuning and evaluation runs were performed through a cloud TPU research programme using TPU v4, v5e, and v6e accelerators.

Each reported training run used a single accelerator, and no reported experiment required multi-node distributed training. Wall-clock time varied substantially with accelerator type, the number of flow steps, the number of sampled quantiles, batch size, evaluation frequency, and whether Monte Carlo rollout generation was required. Final offline RL runs generally required substantially more compute than fixed-policy critic-evaluation runs because they jointly train the critic and policy and include periodic environment evaluation. Across the final reported offline RL experiments, single-seed runs typically completed within 1.5 to 6 accelerator-hours, depending on the configuration and hardware.

Exploratory development used additional compute for debugging, preliminary experiments, and method iteration. We do not report a precise total for this exploratory compute because it was not logged consistently and included many partial or failed runs.

\end{document}